\documentclass{article}

\usepackage[final]{neurips_2025}

\usepackage{amsmath}  % Required for theorem-like environments
\usepackage{amssymb}
\usepackage{mathtools}
\usepackage{algorithm}
\usepackage{algpseudocode}
\usepackage{amsthm}
\usepackage{bbm}
\usepackage{subcaption} 
\usepackage{comment}
\usepackage{cancel}
\usepackage{xcolor}

\addtocontents{toc}{\protect\setcounter{tocdepth}{0}}

\theoremstyle{plain}
\newtheorem{theorem}{Theorem}[section]
\newtheorem{proposition}[theorem]{Proposition}
\newtheorem{lemma}[theorem]{Lemma}
\newtheorem{corollary}[theorem]{Corollary}
\theoremstyle{definition}
\newtheorem{definition}[theorem]{Definition}
\newtheorem{assumption}[theorem]{Assumption}
\theoremstyle{remark}

\usepackage[utf8]{inputenc} % allow utf-8 input
\usepackage[T1]{fontenc}    % use 8-bit T1 fonts
\usepackage{hyperref}       % hyperlinks
\usepackage{url}            % simple URL typesetting
\usepackage{booktabs}       % professional-quality tables
\usepackage{amsfonts}       % blackboard math symbols
\usepackage{nicefrac}       % compact symbols for 1/2, etc.
\usepackage{microtype}      % microtypography
\usepackage{xcolor}         % colors

\title{Near-Optimal Experiment Design in Linear non-Gaussian Cyclic Models}

\author{
  Ehsan Sharifian\\
  EPFL, Lausanne, Switzerland\\ 
  \texttt{ehsan.sharifian@epfl.ch} \\
  \And
  Saber Salehkaleybar \\
  Leiden University, Leiden, The Netherlands\\
  \texttt{s.salehkaleybar@liacs.leidenuniv.nl} \\
  \AND
  Negar Kiyavash \\
  EPFL, Lausanne, Switzerland\\
  \texttt{negar.kiyavash@epfl.ch} \\
}

\newcommand{\bx}{\mathbf{x}}
\newcommand{\be}{\mathbf{e}}
\newcommand{\bq}{\mathbf{q}}
\newcommand{\bp}{\mathbf{p}}
\newcommand{\balpha}{\boldsymbol{\alpha}}

\newcommand{\ce}{\mathcal{E}}
\newcommand{\cs}{\mathcal{S}}
\newcommand{\cp}{\mathcal{P}}

\newcommand{\bbp}{P}
\newcommand{\R}{\mathbb{R}}
\newcommand{\doo}{\mathrm{do}}
\newcommand{\dom}{\mathrm{dom}}
\newcommand{\supp}{\mathrm{supp}}
\newcommand{\diag}{\mathrm{diag}}

\begin{document}

\maketitle

\begin{abstract}
We study the problem of causal structure learning from a combination of observational and interventional data generated by a linear non-Gaussian structural equation model that might contain cycles. Recent results show that using mere observational data identifies the causal graph only up to a permutation-equivalence class. We obtain a combinatorial characterization of this class by showing that each graph in an equivalence class corresponds to a perfect matching in a bipartite graph. This bipartite representation allows us to analyze how interventions modify or constrain the matchings. Specifically, we show that each atomic intervention reveals one edge of the true matching and eliminates all incompatible causal graphs. Consequently, we formalize the optimal experiment design task as an adaptive stochastic optimization problem over the set of equivalence classes with a natural reward function that quantifies how many graphs are eliminated from the equivalence class by an intervention. We show that this reward function is adaptive submodular and provide a greedy policy with a provable near-optimal performance guarantee. A key technical challenge is to efficiently estimate the reward function without having to explicitly enumerate all the graphs in the equivalence class. We propose a sampling-based estimator using random matchings and analyze its bias and concentration behavior. Our simulation results show that performing a small number of interventions guided by our stochastic optimization framework recovers the true underlying causal structure.
\end{abstract}

\section{Introduction}
Learning causal relationships among variables of interest in a complex system is a central goal in empirical sciences, forming the foundation for prediction,  intervention, and explanation~\citep{pearl2009causality}. These relationships are typically represented by causal graphs in which an edge from variable 
$X$ to variable $Y$ indicates that $X$ is a direct cause of $Y$. Much of the existing literature on causal structure learning assumes that the underlying causal graph is a Directed Acyclic Graph (DAG). However, many natural and engineered systems include feedback mechanisms that give rise to cycles in their causal representations. Such cyclic structures emerge in equilibrium models, low-frequency temporal sampling of dynamical systems, and various biological networks~\citep{bongers2021foundations}.

In acyclic settings, a variety of methods (such as those based on conditional independence tests) allow us to recover the skeleton and orientations of the causal graph from observational data. These methods often fail when applied to graphs with cycles. More specifically, observational data alone does not even suffice for learning the skeleton, let alone orienting the edges of the graph \cite{mokhtarian2023unified}.

Interventions, i.e., actively perturbing the system and observing the resulting distributional changes can allow us to learn the graphs even with cycles. The \emph{experiment design} problem studies how to best design interventions in order to maximize the information gained about the causal structure. Again, unlike the case of DAGs, where a body of work on experiment design exist, the work in the cyclic setting \cite{mokhtarian2023unified} is few and far in between. 
This is in part due to the fact that, unlike in DAGs, where an intervention on a subset of the vertices orients all the edges incident to them, in cyclic
directed graphs, performing experiments in some cases, neither leads to learning the presence of edges nor orienting them \citep{mokhtarian2023unified}.

In this work, we study the problem of causal structure learning from a combination of observational and interventional data in systems governed by linear non-Gaussian structural causal models (SCMs) that may contain \textit{cycles}. Our main contributions are as follows: 

\begin{itemize}
    \item We establish that, under linear non-Gaussian assumptions, the causal graph can be identified from observational data only up to an equivalence class. We further provide a combinatorial characterization of the equivalence class,  where each graph corresponds to a perfect matching in a bipartite graph (Section \ref{sec:matching}). Additionally, we show that the condensation graph, or Strongly Connected Components (SCCs), of the true causal graph can also be identified (\ref{sec: SCC and condensation graph}).

    \item The bipartite representation allows us to analyze how atomic interventions constrain the space of causal graphs. Specifically, we  show that each intervention reveals one edge of the true matching and eliminates all incompatible graphs from the equivalence class (Section \ref{sec: Interventional Distribution Information}).
    Therefore, we can formulate the optimal experiment design problem as an adaptive stochastic optimization over the space of equivalence classes, using a reward function that quantifies the number of eliminated graphs following an intervention. We prove that this reward function is adaptive submodular and hence a greedy policy has  provably near-optimal performance guarantees for intervention design (Section \ref{sec: Experiment Design for Causal Structure Learning}).
    \item To address the computational challenge of reward estimation, we propose a sampling-based estimator based on random matchings and provide a theoretical analysis of its bias and concentration properties (Section \ref{sec: Estimating the Reward Function by Sampling}).
    \item Experiments show that our adaptive strategy outperforms other heuristic methods and closely matches the feedback vertex set (FVS) lower bound (see more details in Section~\ref{sec:strategy-comparison}).

\end{itemize}

\section{Related Work}
Causal discovery is concerned with learning the underlying causal graph which encodes both the existence and direction of edges among variables of interest in a system. From observational data alone, the causal graph can be recovered only up to its Markov equivalence class (MEC). \cite{richardson2013polynomial} provided necessary and sufficient conditions for Markov equivalence of directed graphs (DGs) and proposed an algorithm for structure learning up to the MEC \cite{richardson2013discovery}. \cite{mooij2020constraint} extended the applicability of a classic algorithms for constraint-based causal
discovery (i.e., FCI \cite{spirtes2000causation}) to cyclic DGs and showed that it can recover the structure up to the MEC in this more general setting.

To address latent confounding and nonlinear mechanisms, \cite{forre2018constraint} introduced $\sigma$-connection graphs ($\sigma$-CGs)—a flexible class of mixed graphs—and developed a discovery algorithm that handles both latent variables and interventional data. \cite{ghassami2020characterizing} focused on distributional equivalence for linear Gaussian models, providing necessary and sufficient conditions and proposing a score-based learning method. In the context of linear non-Gaussian models, \cite{lacerda2008discovering} extended the Independent Component Analysis (ICA)-based approach of \cite{shimizu2006linear} to cyclic graphs to learn the equivalence class from the observational data.

The research on experiment design in acyclic models has focused on various aspects, including cost minimization and  efficient learning under budget constraints.
 \cite{eberhardt2005number} provided worst-case bounds on the number of experiments to identify the graph. \cite{he2008active} introduced adaptive and exact non-adaptive algorithms for singleton interventions. \cite{hauser2014two} proposed optimal one-shot and adaptive heuristics to minimize undirected edges. \cite{shanmugam2015learning} offered lower bounds for experiment design using separating systems, and \cite{kocaoglu2017experimental} introduced a stage-wise approach in the presence of latent variables.
In cost-aware settings, \cite{kocaoglu2017cost} developed an optimal algorithm for variable-cost interventions without size constraints. \cite{greenewald2019sample} and \cite{squires2020active} designed approximation algorithms for trees and general graphs, respectively.

Fixed-budget design was initiated by \cite{ghassami2018budgeted} with a greedy Monte Carlo-based approach. \cite{ghassami2019counting} improved MEC sampling efficiency, and \cite{ghassami2019interventional} developed an exact experiment design algorithm for tree structures. \cite{ahmaditeshnizi2020lazyiter} proposed an exact method to enumerate DAGs post-intervention. \cite{wienobst2021polynomial,wienobst2023polynomial} showed that MEC counting and sampling can be done in polynomial time.
Bayesian methods include the adaptive submodular algorithm \cite{agrawal2019abcd} and its extension \cite{tigas2022interventions} to optimize both intervention targets and their assigned values.

To the best of our knowledge, only a handful of work on experiment design in cyclic graphs exists. For general SCMs, \cite{mokhtarian2023unified} proposed an experiment design approach that
can learn both cyclic and acyclic graphs. They provided a lower bound on the number of experiments required
to guarantee the unique identification of the causal graph in the worst case, showing that
the proposed approach is order-optimal in terms of the number of experiments up to an additive logarithmic term.
Our approach differs  from \cite{mokhtarian2023unified} in three key aspects. First, their framework is non-adaptive, that is, all interventions are selected in advance based solely on observational data. In contrast, our method is adaptive as each experiment is selected based on the outcomes of prior interventions. Second, \cite{mokhtarian2023unified} allow intervening on multiple variables in a single experiment. Although they also considered the setup where the number of interventions per experiment is bounded, this bound should be greater than the size of the largest strongly connected component in the graph minus one. By contrast, our framework limits each experiment to a single-variable intervention, making it more practical in settings where fine-grained or limited interventions are preferred. Third, the approach in  \cite{mokhtarian2023unified} is designed for general SCMs and relies on conditional independence testing to recover the causal structure. In our work, we focus specifically on linear non-Gaussian models, which enables us to use results from ICA to infer the underlying graph.

\section{Preliminaries and Problem Formulation}
\label{sec: Problem Formulation}
Consider a \emph{Structural Causal Model} (SCM) \(\mathfrak{C} := (\mathbf{S}, P_{\be})\), where \(\mathbf{S}\) consists of \(n\)  structural equations and \(P_{\be}\) is the joint distribution over exogenous noise terms \citep{peters2017elements}. In a linear SCM (LSCM), the structural equations take the following form:
\begin{equation}
    \label{eq:01}
    \bx = W\bx + \be,
\end{equation}
where \(\bx \in \mathcal{X} \subseteq \mathbb{R}^n\) denotes the vector of observable variables, \(\be \in \mathcal{E} \subseteq \mathbb{R}^n\) represents the vector of independent exogenous noises, and \(W \in \mathbb{R}^{n \times n}\) is the matrix of linear coefficients.
The directed graph \(G = (V, E)\) induced by the SCM captures the causal structure: node \(i\) corresponds to the variable \(x_i\), and for any nonzero coefficient \(W_{ij} \neq 0\), a directed edge \((j, i) \in E\) is drawn. The binary adjacency matrix \(B_G \in \{0,1\}^{n \times n}\) is defined as \([B_G]_{ij} = \mathbbm{1}_{\{W_{ij} \neq 0\}}\), and the “free parameters” are the nonzero entries \(\mathrm{supp}(W) = \{(i, j) : W_{ij} \neq 0\}\). We impose the following assumptions:

\begin{assumption}
\label{ass: no self-loop}
The model has no self-loops: \(W_{ii} = 0\) for all \(i \in [n]\). \\
\textit{Remark.} This assumption can be made without loss of generality, as any self-loop can be algebraically removed by reparametrizing the structural equations. For a detailed explanation, see~\cite{lacerda2012discovering}.
\end{assumption}

\begin{assumption}
\label{ass: invertible}
The matrix \(I-W \) is invertible.
\end{assumption}

\begin{assumption}
\label{ass: non-gaussian}
The components of \(\be\) are jointly independent, and at most one component is Gaussian. That is,
\[
P_{\be} \in \mathcal{P}(\mathcal{E}) := \left\{P_{\be} : P_{\be} = \prod_{i=1}^n P_{e_i}, \text{ with at most one } P_{e_i} \text{ Gaussian} \right\}.
\]
\end{assumption}

Under these assumptions, the model admits a unique solution:
\(\bx = (I - W)^{-1} \be = A \be\). %where we denote \(A = (I - W)^{-1}\). 
The observational distribution \(P_{\bx}\) is thus given by the push-forward measure
\(P_{\bx} = (I - T_W)^{-1}_{\#}(P_{\be}),\)
where \(T_W\) denotes the linear map associated with \(W\).

We consider \emph{perfect interventions}, in which the causal mechanism for a variable \(x_i\) is replaced with an exogenous noise \(\tilde{e}_i\), removing all incoming edges to \(x_i\). The interventional SCM is given by:
\begin{equation}
    \label{eq:02}
    \bx = W^{(i)} \bx + \be^{(i)},
\end{equation}
where the \(i\)-th row of \(W^{(i)}\) is zeroed out and the noise in the  \(i\)-th row of $\be$ is replaced with the new independent noise \(\tilde{e}_i\). We denote the new noise vector by  \(\be^{(i)}\). While \((I - W^{(i)})\) might not always be invertible, we assume that the probability it becomes singular is zero under mild randomness conditions on the rows of \(W\). Thus, we treat \((I - W^{(i)})\) as invertible so that the interventional distribution 
\(P(\bx | \doo(x_i)) = (I - T_{W^{(i)}})^{-1}_{\#}(P_{\be^{(i)}})\) is well defined.

In interventional structure learning, we perform \(K\) distinct experiments to reduce uncertainty over the underlying causal graph. Each experiment involves a \emph{perfect intervention} on a single variable. We denote the set of intervention targets by \(\mathcal{I} = \{i_1, i_2, \dots, i_K\} \subseteq [n] \), where \( i_j \in [n] \) indicates that the \(j\)-th experiment intervenes on variable \( x_{i_j} \). In the single-intervention setting we consider, each experiment replaces the structural assignment of the target variable with a new exogenous noise term (cf.\ Equation~\ref{eq:02}). As a result, data from the \(j\)-th experiment is drawn from the interventional distribution \( P(\bx | \doo(x_{i_j})) \). 

Each such experiment provides partial information about the true structural matrix \(W\) by revealing the causal parents of the intervened variable. The overall objective is to combine observational data with strategically designed interventions to eliminate incorrect graph candidates and recover the true causal structure with minimal experimentation.

Before formulating an optimization strategy for experiment design, we first investigate the extent to which the observational distribution constrains the underlying causal graph and matrix \(W\). This characterizes the \emph{observational equivalence class} of graphs, that is, the set of causal structures that cannot be distinguished based on observational data alone.

\section{Distribution Equivalent Graphs}
\label{sec: Distribution Equivalent Graphs}

We aim to characterize the class of linear structural causal models (LSCMs) and their associated graphs that generate the same distribution over observable variables. There are two related but distinct notions in the literature:

\begin{itemize}
    \item \textbf{Distribution-entailment equivalence}, which considers whether two parameterized LSCMs yield the same observational distribution.
    \item \textbf{Distribution equivalence}, which concerns whether two graph structures admit the same set of observable distributions across all compatible LSCMs.
\end{itemize}

These concepts are formally defined and compared in Appendix~\ref{app:dist-equivalence}. In our setting—linear SCMs with non-Gaussian noise—the two notions coincide due to identifiability guarantees from ICA. For a detailed discussion, see also Lacerda et al.~\cite{lacerda2012discovering}.

\vspace{0.5em}
\noindent\textit{In the remainder of this section, we focus on a graph operation that connects distribution-equivalent graphs.}

\begin{definition}[Cycle Reversion]
    \label{def:03}
    Let \(C\) be a cycle in the graph \(G\). The \textbf{cycle reversion (CR)} operation involves swapping the rows of each member of \(C\) in \(I + B_{G}\) with the row corresponding to its subsequent node in the cycle. This operation reverses the direction of the cycle \(C\). Moreover, any edge from a node outside \(C\) that was originally a parent of some \(X \in C\) will now point to the predecessor of \(X\) in the original cycle \(C\) (see Figure~\ref{fig:cycle-reversion}).
\end{definition}

Propositions~\ref{prop:1} and~\ref{prop:2} (see Appendix~\ref{app:dist-equivalence}) shows that distribution-equivalent graphs can be transformed into one another through a sequence of cycle reversions. 

\subsection{Strongly Connected Components}
\label{sec: SCC and condensation graph}
In this section, we examine an important property that remains invariant under cycle reversion. To establish this, we first define strongly connected components (SCCs) in directed graphs.

\begin{definition}[Strongly Connected Component (SCC)]\label{def: SCC}
Two vertices \( u \) and \( v \) in a directed graph \( G \) are said to be strongly connected if there are directed paths from \( u \) to \( v \) and from \( v \) to \( u \). This defines an equivalence relation on the vertex set, whose equivalence classes are called \emph{strongly connected components} (SCCs). Each SCC is a maximal subgraph in which every pair of vertices is strongly connected, and no additional vertex from \( G \) can be added without violating this property. The collection of SCCs forms a partition of the vertex set of \( G \). See~\cite [Section 22.5]{cormen2022introduction}
\end{definition}

\begin{definition}[Condensation Graph]\label{def: condensation graph}
    The \textit{condensation graph} of a directed graph \( G \) is obtained by contracting each SCC into a single vertex and mapping all edges between SCCs to a single edge. The resulting graph is a directed acyclic graph (DAG).
\end{definition}

\begin{figure}[t]
    \centering
    \begin{subfigure}{0.55\linewidth}
        \centering
        \includegraphics[width=\linewidth]{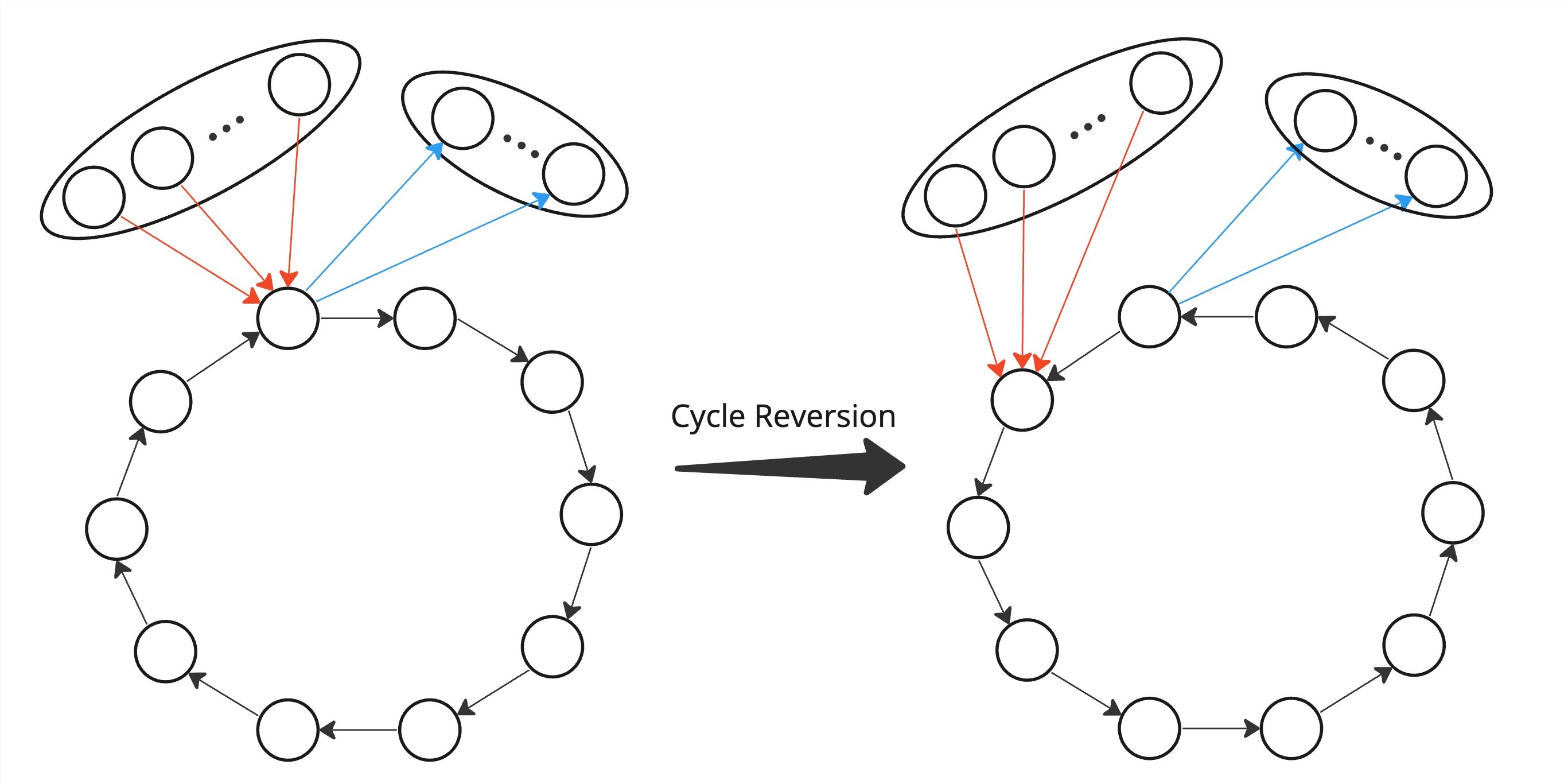}
        \caption{Cycle reversion operation}
        \label{fig:cycle-reversion}
    \end{subfigure}
    \hfill
    \begin{subfigure}{0.33\linewidth}
        \centering
        \includegraphics[width=\linewidth]{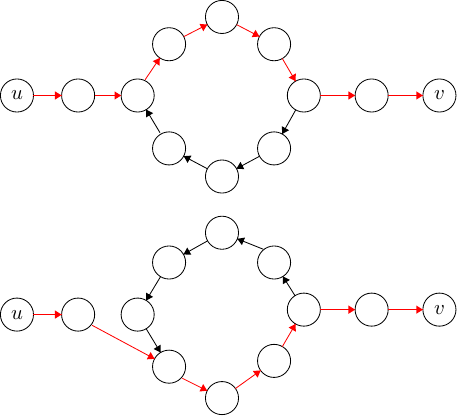}
        \caption{SCCs are preserved after CR}
        \label{fig:scc-maintained}
    \end{subfigure}
    \caption{(a) Graph representation of cycle reversion. (b) The figure highlights representative nodes from an SCC to show that their reachability is maintained, although the full SCC is not depicted.}
    \label{fig:cr-and-scc}
\end{figure}

\begin{theorem}\label{thm: SCC}\footnote{All proofs are provided in the appendix.}
    Strongly connected components remain unchanged in distribution-equivalent graphs.
\end{theorem}

%\begin{proof}
%\textcolor{red}{Deferred to Appendix~\ref{app:thm-SCC-proof}.}
%\end{proof}
Figure~\ref{fig:scc-maintained} illustrates the core idea behind the proof by showing how cycle reversion operations preserve strongly connected components.
The following corollary is a direct consequence of the above theorem and the definition of the condensation graph in \ref{def: condensation graph}.

\begin{corollary}
    The condensation graphs of all distribution-equivalent graphs are identical.
\end{corollary}

%\begin{proof}
%    Distribution-equivalent graphs can be transformed into one another through \textit{CR} operations. According to Theorem \ref{thm: SCC}, \textit{CR} operations do not alter strongly connected components (SCCs). Consequently, if each \textit{SCC} in any of the distribution-equivalent graphs is contracted into a single vertex, they all result in the same condensation graph.
%\end{proof}

Additionally, note that for a subset of vertices in the condensation graph that correspond to single vertices in the original graph, their corresponding row in the recovered matrix via \textit{ICA} can be identified. This is because cycles only occur within an SCC, and if an SCC consists of a single vertex, no cycle passes through it, eliminating any \textit{CR} ambiguity. Therefore, the corresponding row can be determined. Furthermore, if the original graph is acyclic, all its SCCs contain only one vertex, meaning that row permutation ambiguity is entirely resolved. As a result, the original data-generating graph can be inferred purely from observational data, a result previously established in Linear Non-Gaussian Acyclic Models (LiNGAM) in~\cite{shimizu2006linear}.

\subsection{Matching in Bipartite Graphs}
\label{sec:matching}

Graphs in a distribution-entailment equivalence class correspond to perfect matchings in a bipartite graph. Given \( I - W_{\mathrm{ICA}} \), recovered via ICA from observational data, any equivalent graph can be written as \( P_\pi D (I - W_{\mathrm{ICA}}) \), where \( P_\pi \) is a row permutation matrix such that its diagonal entries remain nonzero.

Define a bipartite graph \( G_b = (\{r_1, \dots, r_n\}, [n], E) \), where each node \( r_i \) represents the \( i \)-th row of the matrix \( I - W_{\mathrm{ICA}} \), and an edge \( (r_i, j) \in E \) exists if and only if the $j$-th entry of row $r_i$ is non-zero. Each valid permutation \( \pi \) defines a perfect matching \( r_i \mapsto \pi(i) \), and vice versa. Thus, distribution-equivalent graphs correspond to row permutations consistent with matchings in \( G_b \) (see an example in Appendix~\ref{app:matching-example}).

This bipartite view provides a combinatorial handle on equivalence classes and forms the foundation for analyzing how interventions restrict the matching space.

\section{Interventional Distribution}
\label{sec: Interventional Distribution Information}

We now investigate how interventional data, when combined with observational data, provides additional knowledge about the causal structure. In particular, we show that an intervention on a single variable enables the identification of its causal parents.

\begin{proposition}
\label{prop: intervention information}
Given the observational distribution generated by the model in ~\eqref{eq:01} and the interventional distribution from model~\eqref{eq:02}, the \( i \)-th row of the matrix \( W \) can be recovered. Consequently, the parents of \( x_i \) and their causal effects are identifiable.
\end{proposition}

%\begin{proof}
%\textcolor{red}{Deferred to Appendix~\ref{app:prop-intervention-inf-proof}.}
%\end{proof}

\subsection{Recovering an Edge of the True Matching in the Bipartite Graph}

Proposition~\ref{prop: intervention information} has a natural interpretation in the bipartite graph of \ref{sec:matching}. Recall that each perfect matching in the bipartite graph \( G_b = (\{r_1, \dots, r_n\}, [n], E) \), with edges \( (r_i, j) \in E \) iff the $j$-th entry of row $r_i$ is non-zero, corresponds to a valid permutation matrix \( P_\pi \) defining a distribution-equivalent graph.

Intervening on variable \( x_i \), identifies the row corresponding to \( x_i \). In the bipartite graph, this identifies the matching edge \( (r_{\pi^{-1}(i)}, i) \) corresponding to the true permutation. This eliminates ambiguity for row \( r_{\pi^{-1}(i)} \) and restricts the possible options for matching the remaining vertices in the bipartite graph (see an example in Appendix~\ref{app:intervention-matching-example}).

\begin{comment}
    \paragraph{Example.} Suppose ICA applied to observational and interventional data (intervention on $x_4$) yields the following recovered matrices:

\[
P_1(I-W) = \begin{pmatrix}
b_3^T\\
\mathcolor{red}{b_4^T}\\
b_1^T\\
b_3^T\\
b_5^T\\
\end{pmatrix}
\qquad
P_2(I-W^{(4)}) = \begin{pmatrix}
b_3^T\\
b_5^T\\
b_1^T\\
\mathcolor{red}{e_4^T}\\
b_2^T\\
\end{pmatrix}
\]

Here, row \( b_4^T \) appears only in the observational matrix, allowing us to identify it as the row of \( x_4 \). Therefore, the edge \( (r_2, 4) \) in the bipartite matching is revealed. As this process is repeated across multiple variables, more edges in the true matching are revealed, eventually resolving all ambiguity in the permutation \( \pi \).
\end{comment}

 %This enables full recovery of the true causal graph with a minimal number of strategically chosen interventions.

%\textcolor{red}{Needs to be modified \\
%Graphical Interpretation: It breaks all the cycles that go through the variable ...}

%\textcolor{red}{We can mention that here we focused on single intervention at each experiment but, we can do simultaneous interventions on vertices in different SCCs and the result would be the same}

\section{Adaptive Experiment Design}
\label{sec: Experiment Design for Causal Structure Learning}
In previous sections, we showed that observational data alone allows us to recover the causal structure only up to an equivalence class of graphs, denoted by \( \mathcal{G}_{\text{eq}} \). Each member of this class corresponds to a different row permutation of the matrix \( I - W_{\mathrm{ICA}} \), which is itself associated with a valid matching in the bipartite graph introduced in Section~\ref{sec:matching}. %Thus, structural uncertainty is reduced to a combinatorial ambiguity over row permutations or matchings.

To uniquely identify the causal structure, we need to perform interventions. As demonstrated in Proposition~\ref{prop: intervention information}, an intervention on variable \( x_i \) reveals the row corresponding to \( x_i \), thereby identifies the correct matching edge \( (r_{\pi^{-1}(i)}, i) \) in the bipartite graph. Each such intervention removes all graphs from \( \mathcal{G}_{\text{eq}} \) that are incompatible with the identified edge, shrinking the equivalence class.

We now define an \textbf{adaptive} experiment design framework tailored to this setting. Our objective is to strategically select a sequence of intervention targets to maximally reduce the size of the equivalence class \( \mathcal{G}_{\text{eq}} \), with the goal of identifying the true causal graph \( G^* \) using as few interventions as possible.

Let \( \mathcal{I} = \{i_1, i_2, \dots, i_K\} \subseteq [n] \) be the set of selected variables for intervention. For each \( i_j \in \mathcal{I} \), we collect the interventional data from the distribution \( \bx^{(i_j)} \sim P(\bx | \doo(x_{i_j})) \), and perform ICA to estimate the corresponding interventional matrix \( I - W^{(i_j)} \). By comparing this estimate with the observational estimate \(I-W_{\mathrm{ICA}} \), we recover the row of \( x_{i_j} \) and hence the corresponding matching edge. We then eliminate all graphs in \( \mathcal{G}_{\text{eq}} \) that are incompatible with this edge.

Let \( \Omega := \mathcal{G}_{\text{eq}} \) denote the initial equivalence class. For a given intervention set \( \mathcal{I} \subseteq [n] \) and a true graph \( \phi \in \Omega \), let \( \Omega^{(\mathcal{I}, \phi)} \subseteq \Omega \) denote the set of all graphs that are compatible with the result of interventions on variables in  \( \mathcal{I} \). The objective is to find the set $\mathcal{I}$ that minimizes the size of the reduced equivalence class \( |\Omega^{(\mathcal{I}, \phi)}| \).
In the absence of a prior on the true graph \( \phi \in \Omega \), we consider a uniform prior over \( \Omega \). This leads to an expected-size criterion for selecting the next intervention.

\subsection{Adaptive Stochastic Optimization}
\label{sec:adaptive-objective}

Our aim is to eliminate as many candidate graphs as possible \emph{adaptively}, by choosing each intervention based on the outcomes observed so far. Let the unknown true graph be a random variable drawn uniformly from the set of observationally equivalent graphs, i.e., \(
\Phi \sim \mathrm{Unif}(\Omega)
\).

Following the \emph{Adaptive Stochastic Optimization} framework~\citep{golovin2011adaptive}, we define:

\begin{itemize}
    \item A \emph{policy} \(\pi\) is an adaptive strategy that, at each step \(t\), selects an intervention
      $
        I_t = \pi\bigl((i_1,O_1),\dots,(i_{t-1},O_{t-1})\bigr),
      $
      based on the history of past interventions \(i_1,\dots,i_{t-1}\) and their observed outcomes \(O_1,\dots,O_{t-1}\). Here, each outcome \(O_s\) corresponds to the information gained from intervening on variable \(i_s\)—for example, the index \(\pi^{-1}(i_s)\) of the true row recovered via ICA, or equivalently, the bipartite matching edge \((r_{\pi^{-1}(i_s)},i_s)\) that was identified.
    
    \item For a given realization \( \phi \in \Omega \) and policy \(\pi\), let $
        \mathcal{I}(\pi,\phi) = \bigl\{\,i_1,i_2,\dots,i_K\bigr\}
      $
      be the (random) set of variables on which \(\pi\) intervenes before terminating (subject to a budget of \(K\) interventions).
    
    \item Define the \emph{reward function} as
      \begin{align}
        f(\mathcal{I},\phi) := |\Omega| - |\Omega^{(\mathcal{I},\phi)}|,
        \label{eq:reward_function}
      \end{align}
      where \( \mathcal{I} \subseteq [n] \) denotes the set of intervened variables, and \(\Omega^{(\mathcal{I},\phi)} \subseteq \Omega\) is the set of graphs that agree with \(\phi\) on the outcomes of all interventions in \(\mathcal{I}\).
\end{itemize}

The \emph{utility} of a policy \(\pi\) is the expected reward:
\begin{align}
F(\pi) = \mathbb{E}_{\Phi\sim\mathrm{Unif}(\Omega)} \left[ f\bigl(\mathcal{I}(\pi,\Phi), \Phi \bigr) \right].
\label{eq:objective function}
\end{align}

Our goal is to find an adaptive policy \(\pi\), subject to a budget of \(K\) interventions, that maximizes \(F(\pi)\).

\subsection{Adaptive Monotonicity and Submodularity}

We now review some relavent definitions from the adaptive submodularity framework~\citep{golovin2011adaptive}, and show that our reward function satisfies these properties.

\begin{definition}[Universe and Random Realization]
Let \(\Omega\) be the finite set of all possible true graphs (realizations) consistent with the observational data. We treat the true graph as a random variable \(\Phi \sim \mathrm{Unif}(\Omega)\).
\end{definition}

\begin{definition}[Partial Realization]
Let \(\mathcal{E} = \{1,2,\dots,n\}\) denote the set of variables eligible for intervention, which, in most cases, coincides with the full set of observed variables, and let \(O\) denote the set of possible outcomes from each intervention. A \emph{partial realization} \(\psi\) is a function
\[
\psi: \mathrm{dom}(\psi) \rightarrow O, \quad \mathrm{dom}(\psi) \subseteq \mathcal{E},
\]
where \( \psi(i) \in O \) records the observed outcome (e.g., recovered matching edge) for intervention \(i\). We write \( \Phi \sim \psi \) to denote the posterior distribution over realizations conditioned on consistency with all observations in \(\psi\), i.e.,
\[
\Pr[\Phi = \phi | \Phi \sim \psi] \propto
\begin{cases}
1 & \text{if } \phi \text{ agrees with } \psi, \\
0 & \text{otherwise}.
\end{cases}
\]
\end{definition}

\begin{definition}[Conditional Expected Marginal Benefit]
Let \(f:2^{\mathcal{E}} \times \Omega \to \mathbb{R}_{\ge 0}\) be a reward function and let \(\psi\) be a partial realization. For any \(v \notin \mathrm{dom}(\psi)\), define the \emph{conditional expected marginal benefit} as
\begin{align}
\Delta(v | \psi)
\;:=\;
\mathbb{E}\left[ f\bigl(\mathrm{dom}(\psi) \cup \{v\}, \Phi\bigr) - f\bigl(\mathrm{dom}(\psi), \Phi\bigr) \,|\, \Phi \sim \psi \right].
\label{eq:cemg}
\end{align}
\end{definition}

\begin{definition}[Adaptive Monotonicity]
A reward function \(f\) is \emph{adaptive monotone} if
\begin{align}
\Delta(v | \psi) \ge 0
\label{eq:adaptive_monotone}
\end{align}
for every partial realization \(\psi\) and item \(v \notin \mathrm{dom}(\psi)\).
\end{definition}

\begin{definition}[Adaptive Submodularity]
A reward function \(f\) is \emph{adaptive submodular} if, for all partial realizations \(\psi \subseteq \psi'\) and all \(v \notin \mathrm{dom}(\psi')\),
\begin{align}
\Delta(v | \psi) \ge \Delta(v | \psi').
\label{eq:adaptive_submodular}
\end{align}
\end{definition}

\begin{theorem}\label{thm:adaptive}
Let \( f(\mathcal{I}, \phi) \) be the number of graphs eliminated after performing interventions \( \mathcal{I} \) under realization \( \phi \), as defined in Equation~\eqref{eq:reward_function}. Then \( f \) is adaptive monotone and adaptive submodular with respect to the uniform prior distribution over realizations.

Let \( F(\pi) = \mathbb{E}_{\Phi\sim\mathrm{Unif}(\Omega)} \left[ f\bigl(\mathcal{I}(\pi,\Phi), \Phi \bigr) \right] \) denote the expected number of eliminated graphs under an adaptive policy \( \pi \), as defined in Equation~\eqref{eq:objective function}. Then the adaptive greedy policy achieves a \((1 - 1/e)\)-approximation to the optimal value of \( F \).
\end{theorem}

%\begin{proof}
%\textcolor{red}{Deferred to Appendix~\ref{app:adaptive-proof}.}
%\end{proof}

%\textcolor{red}{We can also talk about the minimum number of interventions for the full identification instead of doing budgeted experiments.}

\section{Estimating the Reward Function by Sampling}
\label{sec: Estimating the Reward Function by Sampling}

Based on Theorem~\ref{thm:adaptive}, %we know that the adaptive greedy policy achieves a \((1 - 1/e)\)-approximation to the optimal policy. Hence,
at each step (after observing the results of previous interventions),
we greedily intervene on the variable that maximally reduces the size of the remaining equivalence class. Thus, it is sufficient to compute the conditional expected marginal benefit \(\Delta(v | \psi)\) for each candidate variable \(v \notin \dom(\psi)\).

Let \(\phi\) be any realization consistent with \(\psi\), and define \(N = |\Omega^{(\dom(\psi), \phi)}|\). In the bipartite-matching representation of the problem, the column vertex \(v\) in the bipartite graph may be matched to different rows across the candidate graphs in \(\Omega^{(\dom(\psi), \phi)}\). Concretely, recall that each graph in this equivalence class corresponds to a perfect matching in the bipartite graph \(G_b\), whose left nodes are rows \(r_1, \dots, r_n\), and whose right nodes are column indices \(1, \dots, n\). Before intervening on variable \(v\), the right-hand node \(v\) may be matched to different row nodes \(r_{z_1}, \dots, r_{z_k}\) across the \(N\) graphs. Denote these distinct row indices by \(z_1, z_2, \dots, z_k\). Therefore, the set of candidate edges for node \(v\) is
\[
\{\, (r_{z_i}, v) \; : \; i = 1, \dots, k \,\}.
\]
Each graph \(g \in \Omega^{(\dom(\psi), \phi)}\) selects exactly one of these edges in its matching. Let
\[
n_i = \bigl|\{\,g \in \Omega^{(\dom(\psi), \phi)} \; : \; g \text{ matches } v \text{ to } r_{z_i} \} \bigr|, \qquad \sum_{i=1}^k n_i = N.
\]

Therefore, prior to an intervention, the probability that \(v\) is matched to row \(r_{z_i}\) is \(p_i := n_i/N\). Once we intervene on \(v\) and recover its true row, we identify the correct edge \((r_{z_i}, v)\) and eliminate all graphs that do not contain this edge. If edge \((r_{z_i}, v)\) is observed, exactly \(n_i\) graphs survive, and the number of eliminated graphs is \(N - n_i\). Hence, the conditional expected marginal benefit under \(\psi\) is:
\begin{align}
\Delta(v | \psi)
&= \sum_{i=1}^k p_i (N - n_i)
= N \sum_{i=1}^k p_i (1 - p_i).
\label{eq:delta_recap}
\end{align}

To compute \(\Delta(v | \psi)\) via Equation~\eqref{eq:delta_recap}, we need to know the full equivalence class \(\Omega^{(\dom(\psi), \phi)}\) (This is necessary for evaluating \(n_i\) and \(N\)). Since explicitly enumerating all graphs in this class is computationally infeasible due to the exponential number of perfect matchings, we estimate \(\Delta(v | \psi)\) by sampling from \(\Omega^{(\dom(\psi), \phi)}\).

Let \(\bp = [p_1, \dots, p_k]^T\) and define the normalized marginal benefit as \(L(\bp) := \sum_{i=1}^k p_i(1 - p_i)\). Although \(\bp\) is not directly accessible, we assume a \textbf{sampling oracle} \(\mathcal{S}\) that returns i.i.d.\ samples from \(\bp\). This corresponds to sampling a perfect matching uniformly at random and observing the edge incident to \(v\). A single batch of such samples allows us to estimate \(L(\bp)\) for all intervention candidates.

The existence of a polynomial-time approximate uniform sampler is ensured by the FPRAS of Jerrum, Sinclair, and Vigoda~\citep{jerrum2004polynomial}. While their method offers theoretical guarantees, we also consider faster heuristic alternatives in practice. For the sake of analysis, it suffices to assume access to such a sampler.

Given samples \(X_1, \dots, X_M \sim \bp\), where each \(X_j \in [k]\) indicates the row matched to \(v\), we construct an empirical estimate \(\hat{\bp}\) and use it to approximate \(\Delta(v | \psi)\).

\subsection{Statistical Accuracy of Sampling-Based Estimation}
\label{sec:statistical-estimation}

We now quantify the accuracy of estimating the marginal benefit \( \Delta(v | \psi) = NL(\bp) = N\sum_{i=1}^k p_i(1 - p_i) \) via empirical sampling.

Let \( \bp = (p_1, \dots, p_k) \) be the unknown categorical distribution over candidate rows (i.e., the edges connected to \(v\) in the bipartite graph). Let \( X_1, \dots, X_M \overset{\text{iid}}{\sim} \bp \), and define the empirical distribution:
\begin{equation}
    \hat{\bp} = (\hat{p}_1, \dots, \hat{p}_k), \quad \hat{p}_i := \frac{1}{M} \sum_{j=1}^M \mathbf{1}\{ X_j = i \}.
    \label{eq:p_estimator}
\end{equation}

We use the estimator
\(
    L(\hat{\bp}) = \sum_{i=1}^k \hat{p}_i (1 - \hat{p}_i)
\). The following theorem provides an estimation error bound on $|L(\bp) - L(\hat{\bp})|$.

\begin{theorem}
\label{thm: bounding the estimation error}
Given $M$ i.i.d samples \( X_1, \dots, X_M \overset{\text{iid}}{\sim} \bp \), for the estimsator $L(\hat{\bp})$, we have:
\begin{align}
\label{estimation error}
\mathbb{P}\left[|L(\hat{\bp}) - L(\bp)| \geq \frac{1}{M} + \sqrt{\frac{2}{M} \log\left(\frac{2}{\epsilon}\right)} \right] \leq \epsilon.
\end{align}
In particular, picking \(\epsilon = \frac{1}{M}\) implies that with probability at least \(1 - \frac{1}{M}\), the error satisfies
\[
|L(\hat{\bp}) - L(\bp)| \leq \frac{1}{M} + \sqrt{\frac{2 \log(2M)}{M}} = \mathcal{O}\left(\sqrt{\frac{\log M}{M}}\right) = \widetilde{\mathcal{O}}\left(\sqrt{\frac{1}{M}}\right).
\]
\end{theorem}

\section{Experimental Results}
\label{sec:experiments}

We evaluate the performance of our adaptive experiment design method for identifying causal graphs under the linear non-Gaussian structural causal model assumption. Experiments are conducted on synthetic data generated from Erdős–Rényi random directed graphs, where each possible directed edge (excluding self-loops) is included independently with a fixed probability. To ensure identifiability, we enforce that the resulting matrix \(I-W \) is invertible. We measure the average number of interventions required for full causal identification.

\subsection{Performance with an Ideal ICA Oracle}
To evaluate the intervention strategy on its own, isolating it from statistical estimation errors, we assume access to an ideal ICA procedure that returns a row-permuted and scaled version of the true matrix \(I-W \). The adaptive method uses a bipartite representation of the graph and samples perfect matchings using two modes: \emph{exact}, where all matchings are enumerated, and \emph{sample}, where a fast greedy heuristic is used (\ref{app:implementation-details-idealized setting}). We compare our adaptive strategy against two baselines: \emph{Random}, which selects a target uniformly at random, and \emph{Max Degree}, which chooses the node with the highest degree in the bipartite graph.

As a theoretical baseline, we compute the \emph{feedback vertex set (FVS)} of each true graph. The FVS size represents a fundamental lower bound on the number of interventions required to break all cycles. Although computing the FVS is NP-hard, we solve it exactly to obtain the best possible benchmark.

Figures~\ref{fig:comparison-sample} and~\ref{fig:comparison-exact} show that our adaptive method consistently outperforms the baselines and operates remarkably close to the intractable FVS lower bound, despite having no knowledge of the true graph structure. This demonstrates the near-optimality of our greedy approach in an ideal setting. For more details, please refer to~\ref{app:implementation-details}.

\begin{figure}[t!]
    \centering
    \begin{subfigure}[t]{0.49\linewidth}
        \centering
        \includegraphics[width=\linewidth]{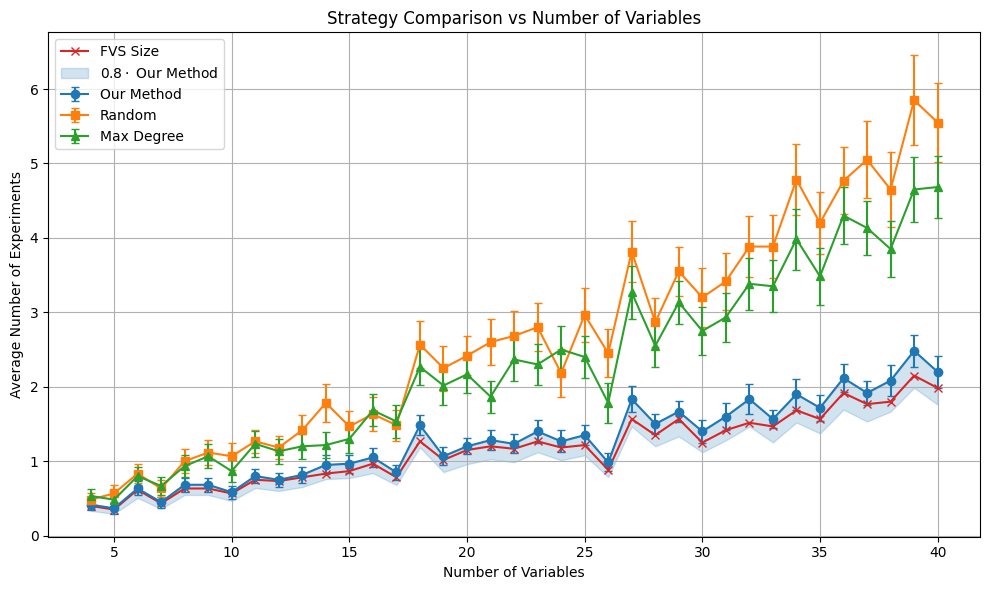}
        \caption{Sample mode: approximate matching sampler.}
        \label{fig:comparison-sample}
    \end{subfigure}
    \hfill
    \begin{subfigure}[t]{0.49\linewidth}
        \centering
        \includegraphics[width=\linewidth]{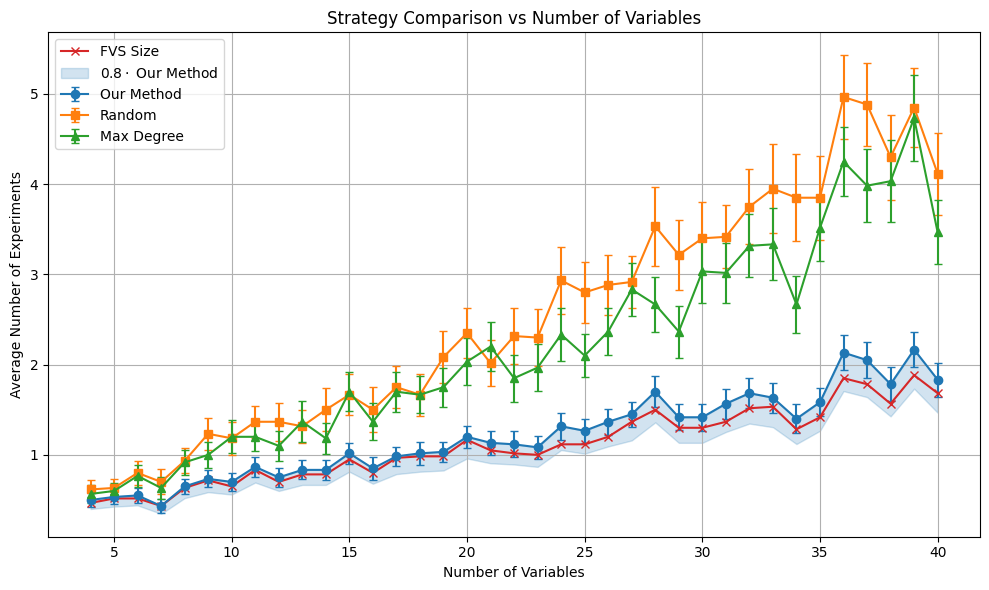}
        \caption{Exact mode: full enumeration of matchings.}
        \label{fig:comparison-exact}
    \end{subfigure}
    \caption{Comparison of intervention strategies assuming an ideal ICA oracle. Our method (Adaptive) consistently performs close to the feedback vertex set (FVS) lower bound.}
    \label{fig:strategy-comparison}
\end{figure}

\subsection{Robustness in a Practical Setting with Finite-Sample ICA}
To assess the practical viability of our method, we evaluate its performance in a more realistic setting without an ideal ICA oracle. Instead, we apply the \texttt{FastICA} algorithm~\cite{hyvarinen2000independent} to finite samples generated from both observational and interventional distributions. This introduces estimation noise, which can corrupt the recovered matrices. To handle these challenges, we introduce two key algorithmic modifications for robustness: \textit{adaptive thresholding} of matrix entries and a \textit{safe matching procedure} to prevent incorrect row assignments. A detailed description of these modifications is provided in Appendix~\ref{app:Implementation-with-Finite-Sample}.

In this setting, our adaptive strategy continues to exhibit a performance trend similar to the ideal case, consistently outperforming the Random and Max Degree baselines and remaining remarkably close to the FVS lower bound (see Figure~\ref{fig:strategy-comparison-FastICA} in Appendix~\ref{app:Implementation-with-Finite-Sample}), confirming its effectiveness even under estimation noise of ICA.

Furthermore, to analyze the quality of the recovered structure, we measured the relative error ($\varepsilon_{\mathrm{rel}}=\frac{\left\|\widehat{W}-W^*\right\|_F}{\left\|W^*\right\|_F}$) between the estimated matrix $\widehat{W}$ and the ground-truth $W^*$. Our result, presented in Appendix~\ref{app:Implementation-with-Finite-Sample}, shows that the algorithm achieves a high-fidelity recovery of the true causal structure in the vast majority of runs.

\section{Extensions}
\label{sec:extensions}

Our core framework focuses on adaptive, single-variable, perfect interventions in linear non-Gaussian models. Here, we briefly discuss how the proposed method can be extended to more general settings. We defer detailed derivations and proofs to Appendix~\ref{app:extensions}.

\subsection{Multi-Node Interventions}
The proposed method can be generalized to handle simultaneous interventions on a set of variables $E = \{i_1, \dots, i_t\}$. Such an intervention is always informative, as it localizes the row-permutation ambiguity to the intervened set $E$ and effectively disambiguates it from the rest of the graph. Furthermore, if the subgraph induced by the variables in $E$ is acyclic, the internal cycle reversion ambiguity within $E$ is fully resolved. This allows for the unique recovery of the true row for each variable in $E$, enabling the safe parallelization of experiments to accelerate the causal discovery process. We provide a detailed analysis in Appendix~\ref{app:extensions-multi-node}.

\subsection{Imperfect Interventions}
Our approach is also robust to imperfect interventions. The core requirement for identifying a variable's corresponding row is that the intervention sufficiently perturbs its causal mechanism. As long as the post-intervention row is distinguishable from the set of observational rows after ICA estimation, the intervention is informative. This condition holds for some realistic noisy or incomplete intervention types, enhancing the practical applicability of our method. Further details are discussed in Appendix~\ref{app:extensions-imperfect}.

\subsection{Generalization to Non-linear Models}
While we focused on the linear case, our experimental design strategy could be extended to non-linear structural causal models. Recent advances in non-linear ICA can recover the Jacobian of the model's inverse function, whose sparsity pattern reveals the underlying causal graph structure up to the same permutation ambiguity found in the linear setting~\cite{reizinger2023jacobian}. Our bipartite matching formulation and reward function could then be applied to resolve this ambiguity. We outline this potential generalization in Appendix~\ref{app:extensions-non-linear}.

\section{Conclusion}

We introduced a framework for causal structure learning in linear non-Gaussian models with cycles, leveraging a combinatorial characterization of equivalence classes via bipartite matchings. By formalizing experiment design as an adaptive submodular optimization problem, we developed a near-optimal greedy policy that incrementally resolves causal ambiguity through targeted interventions. Our sampling-based estimator enables practical implementation without exhaustive enumeration. Empirical results confirm that our method recovers the true graph with few interventions, often matching the feedback vertex set lower bound, despite having no access to the true structure.

\bibliographystyle{plain}
\bibliography{refs}

\appendix

\section*{\LARGE{Supplementary Material}}
\addtocontents{toc}{\protect\setcounter{tocdepth}{2}}
\tableofcontents
\bigskip
\noindent\rule{\textwidth}{0.4pt}
\newpage

\section{Distribution Equivalence: Formal Definitions, Characterization, and Proofs}
\label{app:dist-equivalence}

This appendix provides the formal definitions and results omitted from the main text for brevity. We define two key notions: distribution-entailment equivalence (for parameterized LSCMs) and distribution equivalence (for graphical structures), and show that they are equivalent under the linear non-Gaussian SCM assumption.

\begin{definition}[Distribution Set]
\label{def:01}
Let \( G \) be a directed graph (without self-loops) with binary adjacency matrix \( B_G \in \{0,1\}^{n \times n} \). The \emph{distribution set} \( \cp(G) \) is the collection of all observational distributions \( P_{\bx} \) that can arise from any linear SCM consistent with \( G \):
\begin{align*}
    \cp(G) := \left\{
    P_{\bx} = (I - T_W)^{-1}_{\#}(P_{\be}) \;\middle|\;
    \supp(W) = \supp(B_G),\; P_{\be} \in \cp(\ce)
    \right\},
\end{align*}
where \( T_W \) denotes the linear operator defined by the matrix \( W \), and \( \cp(\ce) \) is the class of exogenous distributions defined in Assumption~\ref{ass: non-gaussian}.
\end{definition}

\begin{definition}[Distribution Equivalence]
\label{def:02}
Two directed \textbf{graphs} \( G_1 \) and \( G_2 \) are said to be \emph{distribution equivalent}, denoted by \( G_1 \equiv G_2 \), if \( \cp(G_1) = \cp(G_2) \).
\end{definition}

Now consider two \emph{parameterized} LSCMs,
\[
\bx = (I - W_1)^{-1} \be_1 \quad \text{and} \quad \bx = (I - W_2)^{-1} \be_2,
\]
with \( P_{\be_1}, P_{\be_2} \in \cp(\ce) \), that induce the same observational distribution. By ICA identifiability, there are some permutation matrix \( P \) and diagonal scaling matrix \( D \) such that:
\begin{align}
    I - W_2 = P D (I - W_1), \quad \be_2 := P D \be_1.
\end{align}
In this case, the two \textbf{models} are said to be \emph{distribution-entailment equivalent}—they have different structure matrices but yield the same distribution over observables.

In the literature, these two concepts are typically treated as distinct. However, in our linear non-Gaussian setting, we show that they coincide. To establish this, we proceed with an algebraic characterization of distribution equivalence.

\begin{proposition}
\label{prop:1}
Two directed graphs (without self-loops) \( G_1 \) and \( G_2 \) are distribution equivalent, that is, \( G_1 \equiv G_2 \), if and only if there exists a permutation matrix \( P \in \{0,1\}^{n \times n} \) such that
\begin{align}
    I + B_{G_2} = P (I + B_{G_1}).
\end{align}
\end{proposition}

This result implies that if two LSCMs are distribution-entailment equivalent, then their associated graphs are distribution equivalent (see the proof~\ref{app:prop1-proof}). Conversely, all models consistent with the same observational distribution must correspond to graphs in the same distribution equivalence class. For linear SCMs with non-Gaussian noise, the two notions of equivalence are therefore interchangeable. 

\vspace{1em}
\noindent
To understand how graphs within the same equivalence class differ, we introduce the cycle decomposition of permutations and relate it to structural changes in the graph. While this idea is discussed in~\cite{lacerda2012discovering}, we include it here for completeness—especially since their treatment is limited to distribution-entailment equivalent LSCMs, rather than distribution-equivalent graphs.

\subsection{Cycle Decomposition and Structural Characterization}

Let \(\sigma: [n] \to [n]\) be a permutation. A \emph{cycle decomposition} of \(\sigma\) expresses it as a product of disjoint cycles. A cycle \((i_1, i_2, \ldots, i_\ell)\) maps \(i_j \mapsto i_{j+1}\) for \(1 \le j < \ell\), and \(i_\ell \mapsto i_1\). Two cycles are disjoint if they share no elements. Every permutation has a unique cycle decomposition (up to cycle order).

Based on Proposition~\ref{prop:1}, if \( G_1 \equiv G_2 \), then there exists a permutation \( P \) satisfying \( I + B_{G_2} = P(I + B_{G_1}) \). The next proposition relates this to cycles in the graph.

\begin{proposition}
\label{prop:2}
Let \(G_1 \equiv G_2\). Let \(P_{\pi}\) be the permutation matrix such that \(I + B_{G_2} = P_{\pi}(I + B_{G_1})\). Suppose \(\pi = c_1c_2\dots c_k\) is its cycle decomposition, where \(c_j = (i_1^j, i_2^j, \ldots, i_{n_j}^j)\) and \(n_j \ge 2\). Then for each \(j \in [k]\),
\begin{itemize}
    \item \( (x_{i^j_1} \rightarrow x_{i^j_2} \rightarrow \cdots \rightarrow x_{i^j_{n_j}} \rightarrow x_{i^j_1}) \) forms a cycle in \(G_1\), and
    \item \( (x_{i^j_1} \leftarrow x_{i^j_2} \leftarrow \cdots \leftarrow x_{i^j_{n_j}} \leftarrow x_{i^j_1}) \) forms a cycle in \(G_2\).
\end{itemize}
\end{proposition}

\noindent\textit{\textbf{Remark.}}
Proposition~\ref{prop:2} highlights that each cycle in the permutation \( \pi \) corresponds to a cycle in one graph whose direction is reversed in the other. However, this characterization does not fully capture the structural relationship between distribution-equivalent graphs. In the main text, we define the \emph{Cycle Reversion} (CR) operation~\ref{def:03}, which provides a more complete description. Unlike the statement of the proposition, the CR operation explicitly specifies how incoming edges from nodes outside the cycle are redirected after reversion. This reveals that the permutation \( \pi \) encodes not only reversed cycles, but also a deeper structural transformation of the graph that preserves distribution equivalence.

\subsection{Proofs of Section \ref{app:dist-equivalence}}
\subsubsection{Proof of Proposition \ref{prop:1}}
\label{app:prop1-proof}
We first show that \( G_1 \equiv G_2 \) implies the existence of such a permutation matrix \( P \).
Suppose \( \cp(G_1) = \cp(G_2) = \cp \), i.e., both graphs admit the same set of observational distributions. Take an arbitrary \( P_{\bx} \in \cp \), then by definition:

\(\bx = (I - W_1)^{-1} \be_1\) with \(\supp(W_1) = \supp(B_{G_1}), \; P_{\be_1} \in \cp(\ce),\)
and also,
\(\bx = (I - W_2)^{-1} \be_2 \) with \(\supp(W_2) = \supp(B_{G_2}), \; P_{\be_2} \in \cp(\ce).\)

From the ICA identifiability result, since \( \bx \) admits two LSCM representations with non-Gaussian independent noises, there must exist a permutation matrix \( P \) and a diagonal scaling matrix \( D \) \(I-W_2 = PD(I-W_1)\) and \(\be_2 = PD\be_1\).

Now, since both graphs lack self-loops, the diagonal entries of \( I - W_2 \) must be 1. We expand this constraint.
Let \( w_i^1 \) be the \( i \)-th row of \( W_1 \), and \( u_j \in \R^n \) be the standard basis row vector with 1 in position \( j \). Then:
\[
P = \begin{bmatrix}
u_{\pi(1)} \\
u_{\pi(2)} \\
\vdots \\
u_{\pi(n)}
\end{bmatrix},
\quad
D = \operatorname{diag}(d_1, d_2, \ldots, d_n),
\]
where \( \pi \) is the permutation induced by \( P \).
Therefore, the matrix product becomes:
\begin{align}
PD(I - W_1) = 
\begin{bmatrix}
d_{\pi(1)} (u_{\pi(1)} - w^1_{\pi(1)}) \\
d_{\pi(2)} (u_{\pi(2)} - w^1_{\pi(2)}) \\
\vdots \\
d_{\pi(n)} (u_{\pi(n)} - w^1_{\pi(n)})
\end{bmatrix}.
\label{eq:PDW1}
\end{align}

To ensure the diagonal entries of \( I - W_2 = PD(I - W_1) \) are 1, we must have, for all \( i \in [n] \),
\begin{align}
d_i = \frac{1}{u_{i, \pi^{-1}(i)} - w^1_{i, \pi^{-1}(i)}} =
\begin{cases}
1 & \text{if } \pi(i) = i, \\
\frac{-1}{w^1_{i, \pi^{-1}(i)}} & \text{if } \pi(i) \neq i.
\end{cases}
\label{eq:di_values}
\end{align}

Hence, \(D\) is uniquely determined by \(P\) and \(W_1\). Since \(\bbp_\bx\) is arbitrary in \(\mathcal{P}\), we can select a distribution generated by \(W_1 = -B_{G_1}\). With this choice we get \(D=I\), yielding \(I-W_2 = P(I+B_{G_1})\). Since \(\text{supp}(W_2) = \text{supp}(B_{G_2})\), we conclude that \(W_2 = -B_{G_2}\).

Now, we will prove the other direction. For every \(\bbp_\bx \in \cp(G_1)\), there is some \(W_1\) and \(\bbp_{\be_1}\in\cp(\ce)\) such that \(\bx = (I-W_1)^{-1}\be_1\) and \(\mathrm{supp}(W_1) = \mathrm{supp}(B_{G_1})\). Define the scaling matrix \(D\) as defined in (\ref{eq:di_values}) based on the given \(P\) in the assumption and \(W_1\). Then define \(W_2 = I - PD(I-W_1)\) and \(\be_2 = PD\be_1\). Applying a scaling and a permutation matrix to an arbitrary matrix does not add or remove zero entries. Therefore, it can be shown that \(\mathrm{supp}(W_2) = \mathrm{supp}(B_{G_2})\). This results in \(\bbp_\bx \in \cp(G_2)\) and then \(\cp(G_1) \subseteq \cp(G_2)\). With the same arguments and substituting \(P\) by \(P^{T}\), we have \(\cp(G_2) \subseteq \cp(G_1)\) and therefore, \(G_1 \equiv G_2\).
\qed

\subsubsection{Proof of Proposition \ref{prop:2}}
\label{app:prop2-proof}
Let \(b_i\) denote the \(i\)-th row of \(B_{G_1}\), and let \(u_j\) be the standard basis row vector in \(\mathbb{R}^n\) with 1 at position \(j\) and 0 elsewhere. As in the proof of Proposition~\ref{prop:1}, the fact that all diagonal entries of \(P_{\pi}(I + B_{G_1})\) equal 1 implies that \(u_{\pi(i), i} + b_{\pi(i), i} = 1\) for all \(i \in [n]\).

If \(i\) is a fixed point, then \(\pi(i) = i\) and the condition becomes \(u_{i, i} + b_{i, i} = 1 + 0 = 1\), since \(G_1\) has no self-loops. Now, let \(i\) belong to a cycle \(c_j = (i^j_1, i^j_2, \ldots, i^j_{n_j})\) of length \(n_j \ge 2\). Then, applying the condition to each \(i^j_r\) gives \(b_{i^j_{r+1}, i^j_r} = 1\) for all \(r = 1, \ldots, n_j\) (with indices modulo \(n_j\)). Thus, the sequence of edges \((x_{i^j_1} \rightarrow x_{i^j_2} \rightarrow \cdots \rightarrow x_{i^j_{n_j}} \rightarrow x_{i^j_1})\) forms a cycle in \(G_1\).

To obtain the analogous result in \(G_2\), observe that \(I + B_{G_1} = P_{\pi^{-1}}(I + B_{G_2})\). Repeating the same argument for \(G_2\) using \(P_{\pi^{-1}}\) implies that \((x_{i^j_1} \leftarrow x_{i^j_2} \leftarrow \cdots \leftarrow x_{i^j_{n_j}} \leftarrow x_{i^j_1})\) forms a cycle in \(G_2\).
\qed

\subsection{Proof of Theorem \ref{thm: SCC}}
\label{app:thm-SCC-proof}
To prove this, we show that the CR operation does not move any vertex from one SCC to another. Since any two distribution-equivalent graphs can be transformed into each other through a sequence of CR operations, this implies that SCCs are preserved.
    
Let \( C = (x_0, x_1, \dots, x_{m-1}) \) be a directed cycle in a graph \( G \), and let \( \cs \) denote the strongly connected component (SCC) that contains all nodes of \( C \). We apply the \emph{Cycle Reversion (CR)} operation to \( C \) and denote the resulting graph by \( G' \). We aim to prove that SCCs remain unchanged under this operation; specifically, no vertex in \( \cs \) exits the component, and no external vertex enters it.

\paragraph{SCCs do not shrink.}
Take any two nodes \( u, v \in \cs \). Since \( \cs \) is strongly connected in \( G \), there exist directed paths \( u \leadsto v \) and \( v \leadsto u \) in \( G \). If these paths do not pass through nodes in \( C \), they remain unaffected in \( G' \).

Now, suppose that the path from \( u \) to \( v \) passes through nodes in \( C \). Without loss of generality, assume the path takes the form:
\[
u \rightarrow p_1 \rightarrow\underbrace{x_{i} \rightarrow x_{i + 1} \rightarrow \cdots \rightarrow x_{i+k}}_{\text{subpath in the cycle } C} \rightarrow p_2 \rightarrow v,
\]
where \( p_1 \) is a subpath from \( u \) to node \( x_{i} \in C \), and \( p_2 \) is a subpath from \( x_{i+k} \in C \) to \( v \) (with indices modulo \( m \)). Before cycle reversion, the edges within \( C \) follow the original cycle direction \( x_0 \rightarrow x_1 \rightarrow \cdots \rightarrow x_{m-1} \rightarrow x_0 \). After applying CR, this direction is reversed: \( x_{m-1} \rightarrow x_{m-2} \rightarrow \cdots \rightarrow x_0 \rightarrow x_{m-1} \) and any edge from a node \( y \notin C \) into a node \( x_j \in C \) is redirected to point to \( x_{j-1} \). Thus, \( p_1 \), which originally led into \( x_{i} \), will now lead to its predecessor \( x_{i-1} \) in the reversed cycle. So we can construct a modified path \( p_1' \) from \( u \) to \( x_{i-1} \).
Moreover, within the reversed cycle, there exists a path from \( x_{i-1} \rightarrow x_{i-2} \rightarrow \cdots \rightarrow x_{i+k} \), corresponding to the reversed traversal of the original path segment in \( C \). Since outgoing edges from nodes in \( C \) to nodes outside are preserved, \( p_2 \) remains valid from \( x_{i+k} \) to \( v \).

Hence, the new path in \( G' \) is:
\[
u \rightarrow p_1' \rightarrow x_{i-1} \rightarrow x_{i-2} \rightarrow \cdots \rightarrow x_{i+k} \rightarrow p_2 \rightarrow v.
\]
By symmetric reasoning, we can similarly construct a path from \( v \) to \( u \) in \( G' \), showing that \( u \) and \( v \) remain strongly connected. Therefore, SCCs do not shrink under CR.

\paragraph{SCCs do not expand.}
Assume for contradiction that a node \( z \notin \cs \) becomes strongly connected with all nodes in \( \cs \) after applying CR, i.e., \( z \in \cs' \), where \( \cs' \) is the SCC containing \( \cs \) in \( G' \). This would mean there exist paths \( z \leadsto u \) and \( u \leadsto z \) in \( G' \) for each \( u \in \cs \). Now, reapplying CR to reverse the cycle \( C \) again (returning to graph \( G \)), the same argument as above shows that SCCs cannot shrink. Thus, \( z \in \cs \) in \( G \) as well, contradicting the assumption \( z \notin \cs \).

\paragraph{Conclusion.}
Since no node leaves or enters \( \cs \) under the CR operation, the strongly connected components remain unchanged. Figure~\ref{fig:scc-maintained} provides a visual illustration of this argument.
\qed

\subsection{Example: Matchings in Bipartite Graphs}
\label{app:matching-example}

As discussed in Section~\ref{sec:matching}, each graph in a distribution equivalence class corresponds to a perfect matching in a bipartite graph defined from the matrix \(I-W_{\mathrm{ICA}} \). The following example illustrates how a specific valid row permutation corresponds to a certain perfect matching. 

\begin{equation*}
I - W_{\mathrm{ICA}} =
\begin{pmatrix}
\circ & \mathcolor{red}{\times} & \times & \times & \circ \\
\mathcolor{red}{\times} & \circ & \circ & \times & \times \\
\circ & \times & \circ & \mathcolor{red}{\times} & \times \\
\times & \times & \mathcolor{red}{\times} & \circ & \circ \\
\times & \circ & \times & \circ & \mathcolor{red}{\times} \\
\end{pmatrix}
\end{equation*}

Here, cross marks indicate non-zero entries in the matrix, and the red entries mark a specific row permutation \( \pi \) that selects a unique non-zero element in each row, corresponding to a perfect matching. This selection defines a valid permutation matrix \( P_\pi \) such that \( P_\pi (I-W_{\mathrm{ICA}}) \) is a valid causal model in the equivalence class with non-zero elements on the diagonal.
Figure~\ref{fig:matching_bipartite} further illustrates how such permutations correspond to bipartite matchings.

\begin{figure}[t]
    \centering
    \begin{subfigure}[t]{0.36\linewidth}
        \centering
        \includegraphics[width=0.55\linewidth]{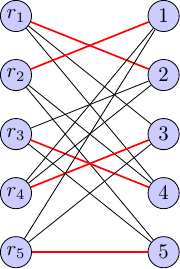}
        \caption{Permutation as a matching.}
        \label{fig:matching_bipartite}
    \end{subfigure}
    \hfill
    \begin{subfigure}[t]{0.58\linewidth}
        \centering
        \includegraphics[width=0.55\linewidth]{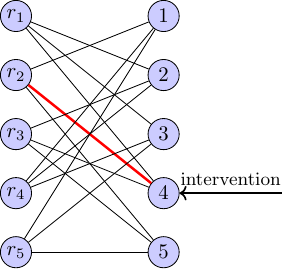}
        \caption{Matching after intervention.}
        \label{fig:matching-after-intervention}
    \end{subfigure}
    \caption{(a) Each row permutation corresponds to a matching between sources and observed variables. (b) An intervention on a specific variable reveals its correct row assignment, narrowing the equivalence class.}
    \label{fig:combined-matching}
\end{figure}

\section{Proofs and Examples of Section \ref{sec: Interventional Distribution Information}}
\subsection{Proof of Proposition \ref{prop: intervention information}}
\label{app:prop-intervention-inf-proof}
Applying ICA to both observational and interventional data yields \( P D (I - W) \) and \( P_i D_i (I - W^{(i)}) \), where \( P, P_i \) are permutation matrices and \( D, D_i \) are diagonal scaling matrices.

Since \( W \) and \( W^{(i)} \) differ only in the \( i \)-th row, and all rows are linearly independent (due to invertibility), their normalized forms will differ in exactly one row. This unique row in \( P D (I - W) \), absent in \( P_i D_i (I - W^{(i)}) \), corresponds to the \( i \)-th row of \( W \), revealing its support and weights. Thus, the valid row permutation must place this row at index \( i \).

\subsection{Example: How Interventions Reveal True Matching Edges}
\label{app:intervention-matching-example}

As discussed in Proposition~\ref{prop: intervention information}, interventions provide critical information about the true matching in the bipartite representation. Recall that each perfect matching in the bipartite graph \( G_b = (\{r_1, \dots, r_n\}, [n], E) \), where an edge \( (r_i, j) \in E \) exists iff \( (I-W_{\mathrm{ICA}})[i,j] \neq 0 \), corresponds to a valid row permutation matrix \( P_\pi \) defining a distribution-equivalent graph.

When we intervene on variable \( x_i \), the corresponding row in \( I - W^{(i)} \) becomes identifiable. In the bipartite graph, this uniquely reveals the edge \( (r_{\pi^{-1}(i)}, i) \), which is part of the true matching. This discovery not only pins down the correct match for one node pair but also restricts the set of consistent permutations, narrowing the equivalence class.

\paragraph{Example.} Suppose ICA applied to observational and interventional data (intervention on \( x_4 \)) yields the following recovered matrices:

\[
P_1(I - W) =
\begin{pmatrix}
r_1^T \\
\mathcolor{red}{r_2^T} \\
r_3^T \\
r_4^T \\
r_5^T \\
\end{pmatrix}
\qquad
P_2(I - W^{(4)}) =
\begin{pmatrix}
r_3^T \\
r_5^T \\
r_1^T \\
\mathcolor{red}{u_4^T} \\
r_4^T \\
\end{pmatrix}
\]

Here, row \( r_2^T \) appears only in the observational matrix, while \( u_4^T = [0, 0, 0, 1, 0]\) appears only in the interventional one. This mismatch allows us to identify \( r_2^T \) as the true row of \( x_4 \), revealing the matching edge \( (r_2, 4) \) in the bipartite graph. As this process is repeated across multiple variables, more edges in the true matching are discovered, eventually resolving the full permutation \( \pi \). Figure~\ref{fig:matching-after-intervention} illustrates this idea of how such interventional data reveals some information about the true hidden perfect matching corresponding to the true causal model.

\section{Proofs of Section \ref{sec: Experiment Design for Causal Structure Learning}}
\subsection{Proof of Theorem \ref{thm:adaptive}}
\label{app:adaptive-proof}
In this appendix, we provide the full proof that the reward function
\[
f(\mathcal I,\phi)\;=\;|\Omega|\;-\;\bigl|\Omega^{(\mathcal I,\phi)}\bigr|
\]
is both adaptive monotone and adaptive submodular.

\subsection*{Adaptive Monotonicity}

By construction, performing an additional intervention can only eliminate more graphs (never re‐introduce any). Hence, for any partial realization \(\psi\) and any \(v\notin\dom(\psi)\),
\[
f\bigl(\dom(\psi)\cup\{v\},\phi\bigr)
\;\ge\;
f\bigl(\dom(\psi),\phi\bigr)
\quad
\text{for every realization } \phi.
\]
Taking expectation conditional on \(\Phi\sim\psi\) preserves the inequality, establishing adaptive monotonicity (inequality~\eqref{eq:adaptive_monotone}).

\subsection*{Adaptive Submodularity}

Fix two partial realizations \(\psi\subseteq\psi'\) with
\[
\mathcal I_1=\dom(\psi),\quad \mathcal I_2=\dom(\psi'),\quad \mathcal I_1\subseteq \mathcal I_2,
\]
and let \(\phi\) be any realization consistent with \(\psi'\).  Define
\[
\Omega_1 \;=\;\Omega^{(\mathcal I_1,\phi)},\quad
\Omega_2 \;=\;\Omega^{(\mathcal I_2,\phi)},
\]
and denote their sizes by
\[
N_1=|\Omega_1|,\quad N_2=|\Omega_2|,
\]
and \(N_2\le N_1\).  Let \(v\notin \mathcal I_2\) be the next intervention candidate.  In the bipartite‐matching view, assume that \(v\) has \(k\) possible edges across \(\Omega_1\). Concretely, recall that each candidate graph in our current equivalence class \(\Omega_1\) corresponds to a perfect matching in the bipartite graph \(G_b\) whose left nodes are rows \(r_1,\dots,r_n\) and whose right nodes are column positions \(1,\dots,n\).  Before intervening on variable \(v\), the right‐hand node \(v\) may be matched to different rows \(r_i\) across the \(\lvert\Omega_1\rvert=N_1\) graphs.  Call these distinct row‐indices \(i=z_1,z_2,\dots,z_k\), so that the set of “candidate edges” for node \(v\) is
\[
\{\, (r_{\,i},v)\;:\; i=z_1,z_2,\dots,z_k\},
\]
and each graph \(g\in\Omega_1\) picks exactly one of these edges in its matching (Figure~\ref{fig: reduction of the equivalence class}). 

\begin{figure}[t]
    \centering
    \includegraphics[width=0.6\linewidth]{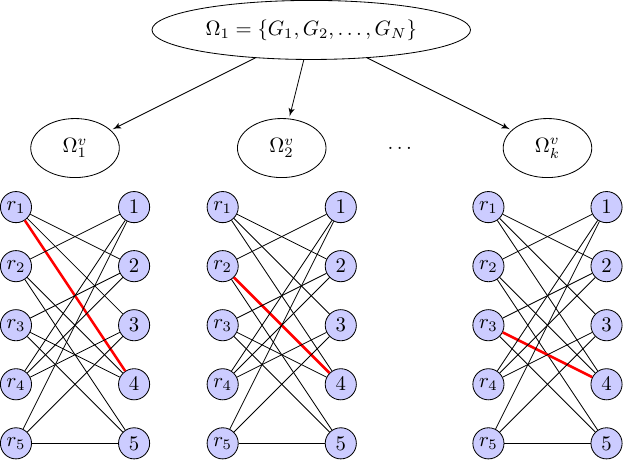}
    \caption{Illustration of equivalence class partitioning induced by a candidate intervention on variable \(v=4\). The top node represents the current equivalence class \(\Omega_1 = \{G_1, \dots, G_N\}\), and each child \(\Omega_i^v\) consists of graphs whose perfect matching assigns variable \(v\) to row \(r_{z_i}\), i.e., includes the edge \((r_{z_i}, v)\). The bipartite graphs below depict three such subsets, each highlighting a different candidate edge \((r_{z_i}, v)\) in red. This illustrates how an intervention on \(v\) resolves ambiguity by eliminating all but one of these subsets, reducing the equivalence class.}
    \label{fig: reduction of the equivalence class}
\end{figure}

Let
\[
n_i \;=\;\bigl|\{\,g\in\Omega_1 : g\text{ matches }v\text{ to }r_{z_i}\}\bigr|\,, 
\qquad
\sum_{i=1}^k n_i = N_1.
\]
Hence prior to the intervention, the probability that \(v\) is matched to row \(r_i\) is 
\(\;p_i = n_i/N_1\).  Once we intervene on \(v\) and recover its true structural row, we immediately identify which one of the \(k\) edges \((r_i,v)\) is correct, and eliminate all graphs in \(\Omega_1\) that did not use that edge.
If edge \(i\) is observed, exactly \(n_i\) graphs survive, so the number of
eliminated graphs is \(N_1-n_i\).  Hence, the conditional expected marginal
benefit under \(\psi\) is
\begin{align}
\Delta(v|\psi)
&=\sum_{i=1}^k p_i\,(N_1-n_i)
=\sum_{i=1}^k \frac{n_i}{N_1}(N_1-n_i)
= N_1\sum_{i=1}^k p_i(1-p_i).
\label{eq:delta1_recap}
\end{align}

Similarly, under the  realization \(\psi'\), only the subset
\(\Omega_2\subseteq\Omega_1\) remains.  Let
\[
m_i=\bigl|\{\,g\in\Omega_2: g\text{ matches }v\text{ to }r_{z_i}\}\bigr|,\quad
\sum_{i=1}^k m_i=N_2,
\]
and set \(q_i=m_i/N_2\).  Then
\begin{align}
\Delta(v|\psi')
= N_2\sum_{i=1}^k q_i(1-q_i).
\label{eq:delta2_recap}
\end{align}

We now show
\[
\Delta(v|\psi)\;-\;\Delta(v|\psi')
\;=\;
N_1\sum_{i=1}^k p_i(1-p_i)
\;-\;
N_2\sum_{i=1}^k q_i(1-q_i)
\;\ge\;0.
\]
To see this, note that transition from \(\Omega_1\) to \(\Omega_2\) can be
viewed as \textbf{removing} one graph in $\Omega_1\backslash \Omega_2$ at a time.  We will verify that removing a
single graph from a collection of size \(N\) can only \textbf{decrease}
\(\;N\sum_i p_i(1-p_i)\).

\paragraph{Removal of a single graph.}
Suppose we remove one graph of “type” \(j\) from a collection of size
\(N_1\).  Then
\[
n'_j=n_j-1,\quad n'_i=n_i\;(i\neq j),\quad N'=N_1-1,
\]
and the new probabilities are
\[
p'_j=\frac{n_j-1}{N'},\quad
p'_i=\frac{n_i}{N'}\;(i\neq j).
\]
We compare
\[
M(\Omega_1)
= N_1\sum_i p_i(1-p_i)
\quad\text{and}\quad
M(\Omega')
= N'\sum_i p'_i(1-p'_i).
\]
Let \(S=\sum_i n_i^2\).  Then
\[
M(\Omega_1)
= N_1\Bigl(1-\sum_i p_i^2\Bigr)
= N_1-\frac{S}{N_1},
\]
and
\[
M(\Omega')
= N'-\sum_i\frac{(n'_i)^2}{N'}
= (N_1-1)
- \frac{\,\bigl(S - n_j^2 + (n_j-1)^2\bigr)\!}{N_1-1}
= (N_1-1)
- \frac{S - 2n_j +1}{N_1-1}.
\]
Hence
\[
\begin{aligned}
M(\Omega_1)-M(\Omega')
&=\Bigl(N_1-\tfrac{S}{N_1}\Bigr)
- \Bigl((N_1-1)-\tfrac{S -2n_j +1}{N_1-1}\Bigr)\\[-4pt]
&=1-\frac{S}{N_1}+\frac{S-2n_j+1}{N_1-1}\\
&=\frac{N_1(N_1-1)-S(N_1-1)+N_1(S-2n_j+1)}{N_1(N_1-1)}\\
&= \frac{N_1^2 - 2n_jN_1 + S + \left(n_j^2 - n_j^2\right)}{N_1(N_1-1)}\\
&=\frac{(N_1-n_j)^2+\sum_{i\neq j}n_i^2}{N_1(N_1-1)}\;\ge\;0.
\end{aligned}
\]
Thus, removing a single graph cannot increase \(M\).  
\paragraph{Induction to multiple removals.}
To reach \(\Omega_2\) of size \(N_2\), we remove \(N_1-N_2 \) graphs
one at a time.  Each removal step leaves \(M\) unchanged or decreases it.
Hence
\[
N_1\sum_i p_i(1-p_i)
\;\ge\;
N_2\sum_i q_i(1-q_i)
\quad\Longrightarrow\quad
\Delta(v|\psi)\;\ge\;\Delta(v|\psi'),
\]
establishing adaptive submodularity \eqref{eq:adaptive_submodular}.  

By the result of Golovin and Krause~\citep{golovin2011adaptive}, adaptive monotonicity and submodularity guarantee that the adaptive greedy policy achieves at least a \((1 - 1/e)\)-approximation to the optimal policy. 
\qed

\subsection*{Alternative \(\boldsymbol{\alpha}\)–Parameterization Proof}

An equivalent way to see that conditioning (going from \(\psi\) to \(\psi'\)) only \emph{decreases} the marginal benefit is via the following re-parameterization. We present this alternative \(\boldsymbol{\alpha}\)–parameterization proof because it is more general, as it does not rely on the discreteness of variables. Unlike previous approaches, this proof remains valid even when the variables are continuous, thereby broadening the scope of applicability to a wider range of models.

Fix a partial realization \(\psi\) so that \(\Omega_1=\Omega^{(\dom(\psi),\phi)}\) has size \(N\) and assign to the next intervention \(v\) the matching edge $(r_{z_j},v)$ with probability
\[
p_j \;=\;\frac{n_j}{N}, 
\]
where \(n_j\) of the \(N\) surviving graphs in \(\Omega_1\) match \(v\) to row \(r_j\).  After further conditioning under \(\psi'\), only \(\Omega_2\subseteq\Omega_1\) of size \(N'\le N\) remains, and among those \(m_j\le n_j\) graphs still use edge \(j\).  Define
\[
\alpha_j \;=\;\frac{m_j}{n_j}\in[0,1], \quad Z \;=\;\sum_{i=1}^k p_i\,\alpha_i.
\]

Conditioned on \(\psi'\), the probability that we observe edge \(j\) is
\[
\frac{m_j}{\sum_{i}m_i} = \frac{\alpha_jn_j}{\sum_{i}\alpha_in_i} = \frac{N\alpha_jp_j}{N\sum_{i}\alpha_ip_i} = \frac{\alpha_jp_j}{Z},
\]
and upon observing it we eliminate \(m_j = \alpha_jn_j\)
graphs from $\sum_{i=1}^{k}m_i = ZN$ graphs.  Hence the \emph{conditional expected marginal benefit} under \(\psi'\) is
\begin{align}
\Delta\bigl(v|\psi'\bigr)
&=\sum_{j=1}^k \frac{p_j\,\alpha_j}{Z}\,\bigl(NZ - n_j\,\alpha_j\bigr)
=N\sum_{j=1}^k p_j\,\alpha_j\,\bigl(1 - \frac{p_j\,\alpha_j}{Z}\bigr)
\nonumber\\
&= N\sum_{j=1}^k\frac{p_j\,\alpha_j\,(Z-p_j\alpha_j)}{Z}
\;=\;
\frac{N}{Z}\Bigl(Z^2 - \sum_{j=1}^k p_j^2\,\alpha_j^2\Bigr)
\nonumber\\
&=\;
\frac{N}{Z}\sum_{i\neq j} p_i\,\alpha_i\;p_j\,\alpha_j.
\label{eq:alt-delta-alpha}
\end{align}

Equivalently, in matrix form let \(\bp=(p_1,\dots,p_k)^T\) and
\(\balpha=(\alpha_1,\dots,\alpha_k)^T\).  Then
\(\bp^T(\balpha\balpha^T - \diag(\balpha^2))\bp=\sum_{i\neq j}p_i\alpha_i\,p_j\alpha_j\),
and $Z=\sum_{j=1}^k p_j\,\alpha_j = \bp^T\balpha$, So we may define
\[
f(\,\balpha\,)
\;=\;
\Delta(v|\psi')
\;=\;
N\frac{\bp^T\bigl(\balpha\balpha^T - \diag(\balpha^2)\bigr)\bp}{\bp^T\balpha}.
\]
To prove adaptive submodularity, it suffices to show \(f\) is coordinate-wise
non-decreasing on \([0,1]^k\), since then conditioning from
\(\balpha=\mathbf1\) (smaller partial observation) to any \(\balpha\in[0,1]^k\) (bigger partial observation) can only decrease \(f\) (the marginal benefit of adding \(v\) to the intervention set).

Now, we can compute the partial derivatives as follows
\begin{align}
    \frac{\partial f}{\partial \alpha_i}
=
N\;
\frac{
\Bigl(\frac{\partial}{\partial \alpha_i}\bp^T(\balpha\balpha^T-\diag(\balpha^2))\bp\Bigr)\,Z
-
\bp^T(\balpha\balpha^T-\diag(\balpha^2))\bp\;\frac{\partial Z}{\partial \alpha_i}
}{
Z^2
}.
\label{eq:df-dqi}
\end{align}
One checks easily
\[
\frac{\partial Z}{\partial \alpha_i}=p_i,
\quad
\frac{\partial}{\partial \alpha_i}
\bigl[\bp^T(\balpha\balpha^T-\diag(\balpha^2))\bp\bigr]
=2\,p_i\sum_{j\neq i}p_j\,\alpha_j.
\]

Substituting into \eqref{eq:df-dqi} gives
\begin{align*}
\frac{\partial f}{\partial \alpha_i}
&=
N\;
\frac{
2p_i\bigl(\sum_{j\neq i}p_j\alpha_j\bigr)\,Z
-
\bigl(\sum_{r\neq s}p_r\alpha_r\,p_s\alpha_s\bigr)\,p_i
}{
Z^2
}\\
&=
N\,p_i\;
\frac{2Z\bigl(Z - p_i\alpha_i\bigr) - \bigl(Z^2 - \sum_{j=1}^{k}(p_j\alpha_j)^2\bigr)
}{Z^2}.  \\
&= \frac{Np_i}{Z^2}(Z^2 - 2p_i\alpha_iZ + \sum_{j=1}^{k}(p_j\alpha_j)^2)
\end{align*}
Thus, to show that $\frac{\partial f}{\partial \alpha_i} \geq 0$, one needs to show $Z^2 - 2p_i\alpha_iZ + \sum_{j=1}^{k}(p_j\alpha_j)^2 \geq 0$. It can be shown by some simple algebraic manipulation. By defining $Z_{-i} = \sum_{j\neq i} p_j\alpha_j$, we have
\begin{align*}
    Z^2 - 2p_i\alpha_iZ + \sum_{j=1}^{k}(p_j\alpha_j)^2 &= (Z_{-i} + p_i\alpha_i)^2 - 2p_i\alpha_i(Z_{-i} + p_i\alpha_i) + \sum_{j=1}^{k}(p_j\alpha_j)^2 \\
    &= Z_{-i}^2 + 2p_i\alpha_iZ_{-i} + p_i^2\alpha_i^2 -2p_i\alpha_iZ_{-i}- 2p_i^2\alpha_i^2 + \sum_{j=1}^{k}(p_j\alpha_j)^2 \\
    &= Z_{-i}^2 + \sum_{j\neq i}(p_j\alpha_j)^2 \geq 0
\end{align*}

Hence \(f\) is coordinatewise non‐decreasing on \([0,1]^k\).  In particular, passing from \(\balpha=\mathbf1\) to any $\balpha\in[0,1]^k$ can only decrease $f$, which completes the proof of adaptive submodularity via \(\alpha\)–parameterization.
\qed

\section{Concentration Bounds for the Normalized Marginal Benefit}
\label{app:normalized-marginal-benefit}

This section establishes a concentration bound for the empirical normalized marginal benefit
\[
L(\hat{\bp}) = \sum_{i=1}^k \hat{p}_i (1 - \hat{p}_i),
\]
computed from \(M\) i.i.d.\ samples of a categorical distribution \(\bp \in \Delta_k\). The quantity \(L(\hat{\bp})\) is used in our algorithm as a heuristic to estimate the benefit of intervening on a given variable. We aim to bound the estimation error
\[
\left| L(\bp) - L(\hat{\bp}) \right|.
\]

Using the triangle inequality, we write:
\[
\left| L(\bp) - L(\hat{\bp}) \right| \leq \left| L(\bp) - \mathbb{E}[L(\hat{\bp})] \right| + \left| \mathbb{E}[L(\hat{\bp})] - L(\hat{\bp}) \right|.
\]
This gives rise to a bias term and a deviation term, which we control separately below.

\subsection{Bounding the Bias Term}
\begin{lemma}
\label{lemma: bias term}
Let \( \hat{\bp} \) denote the empirical distribution over \(M\) i.i.d.\ samples from a categorical distribution \( \bp \in \Delta_k \). Then the bias of the estimator satisfies
\[
\left| L(\bp) - \mathbb{E}[L(\hat{\bp})] \right| \leq \frac{L(\bp)}{M} \leq \frac{1}{M}.
\]
\end{lemma}

\begin{proof}
\label{app:bias-control-proof}
We compute the bias directly:
\begin{align*}
    \mathbb{E}[L(\hat{\bp})] &= \mathbb{E}\left[\sum_{i=1}^k \hat{p}_i (1 - \hat{p}_i)\right] \\
    &= \sum_{i=1}^k \mathbb{E}[\hat{p}_i] - \sum_{i=1}^k \mathbb{E}[\hat{p}_i^2].
\end{align*}
Note that \( \mathbb{E}[\hat{p}_i] = p_i \), and:
\begin{align*}
    \mathbb{E}[\hat{p}_i^2] &= \mathbb{E}\left[\left(\frac{1}{M} \sum_{j=1}^M \mathbf{1}\{X_j = i\}\right)^2\right] \\
    &= \frac{1}{M^2}\sum_{j=1}^{M}\sum_{\ell=1}^{M} \mathbb{E}[\mathbf{1}\{X_j = i\} \cdot \mathbf{1}\{X_{\ell} = i\}] \\
    &= \frac{1}{M^2} \sum_{j=1}^M \mathbb{P}[X_j = i] + \frac{1}{M^2} \sum_{j \neq \ell} \mathbb{P}[X_j = i, X_{\ell}=i] \\
    &= \frac{1}{M^2}(M p_i + M(M-1) p_i^2) \\
    &= \frac{p_i}{M} + \left(1 - \frac{1}{M} \right) p_i^2.
\end{align*}
Therefore,
\begin{align*}
    \mathbb{E}[L(\hat{\bp})] &= \sum_{i=1}^k p_i - \sum_{i=1}^k \left( \frac{p_i}{M} + \left(1 - \frac{1}{M} \right) p_i^2 \right) \\
    &= \left(1 - \frac{1}{M}\right) \left(\sum_{i=1}^k p_i - \sum_{i=1}^k p_i^2\right)\\
    &= \left(1 - \frac{1}{M}\right)L(\bp),
\end{align*}
which yields
\begin{align*}
    |L(\bp) - \mathbb{E}[L(\hat{\bp})]| &= \frac{1}{M} L(\bp) \\
    &= \frac{1}{M}\sum_{i=1}^k p_i(1-p_i) \\
    &\overset{*}{\leq} \left(1 - \frac{1}{k}\right)\frac{1}{M} \\
    &\leq \frac{1}{M}.
\end{align*}

\paragraph{Justification for \(\ast\).} The bound follows from the identity
\[
\sum_{i=1}^k p_i(1 - p_i) = \sum_{i=1}^k p_i - \sum_{i=1}^k p_i^2 = 1 - \sum_{i=1}^k p_i^2,
\]
using the fact that \( \sum_{i=1}^k p_i = 1 \) for any \( p \in \Delta_k \). To maximize this expression, we minimize \( \sum_{i=1}^k p_i^2 \) under the simplex constraint. By the Cauchy-Schwarz inequality,
\[
\left( \sum_{i=1}^k p_i^2 \right)\left( \sum_{i=1}^k 1^2 \right) \ge \left( \sum_{i=1}^k p_i \right)^2 = 1^2,
\]
so \( \sum_{i=1}^k p_i^2 \ge \frac{1}{k} \), with equality if and only if \( p_i = \frac{1}{k} \) for all \( i \). Thus, the maximum of the original expression is
\(
1 - \frac{1}{k},
\)
attained uniquely by the uniform distribution.
\end{proof}

\subsection{Bounding the Deviation Term}
To bound the deviation term \( \left| \mathbb{E}[L(\hat{\bp})] - L(\hat{\bp}) \right| \), we invoke the bounded differences inequality. We begin with the following definition:

\begin{definition}
A function \( f: \mathcal{X}^M \rightarrow \mathbb{R} \) satisfies the bounded difference property with parameters \( c_1, \dots, c_M \) if changing the \(i\)-th coordinate of the input changes the function value by at most \(c_i\). That is, for all \( i \in [M] \) and any input vectors \( x_1, \dots, x_M \in \mathcal{X} \), and \( x_i' \in \mathcal{X} \),
\[
\left| f(x_1, \dots, x_i, \dots, x_M) - f(x_1, \dots, x_i', \dots, x_M) \right| \leq c_i.
\]
\end{definition}

A well-known result~\citep{vershynin2018high} establishes concentration of measure for bounded difference functions when their inputs are independent random variables.

\begin{theorem}[Bounded Differences Inequality, Theorem 2.9.1 in~\citep{vershynin2018high}]
\label{thm: concentration inequality for bounded difference functions}
Let \( X = (X_1, \dots, X_M) \) be a vector of independent random variables, and let \( f: \mathcal{X}^M \rightarrow \mathbb{R} \) satisfy the bounded difference property with parameters \( c_i \). Then for any \( t > 0 \),
\[
\mathbb{P}\left[ f(X) - \mathbb{E} f(X) \geq t \right] \leq \exp\left( -\frac{2 t^2}{\sum_{i=1}^M c_i^2} \right).
\]
\end{theorem}

\begin{lemma}
\label{lemma: bounded difference property}
Let \( \hat{\bp} \) denote the empirical distribution over \(M\) i.i.d.\ samples \( X_1, \dots, X_M \in [k] \), and define \(  L(X_1, \dots, X_M) = L(\hat{\bp}) = \sum_{i=1}^k \hat{p}_i(1 - \hat{p}_i) \). Then the function \( L(X_1, \dots, X_M): [k]^M \rightarrow \mathbb{R} \) satisfies the bounded difference property with parameters \( c_i = \frac{2}{M} \) for all \( i \in [M] \). Based on theorem~\ref{thm: concentration inequality for bounded difference functions}, it implies the following concentration inequality:
\[
\mathbb{P}\left[| L(\hat{\bp}) - \mathbb{E} [L(\hat{\bp})]| \geq t \right] \leq 2\exp\left( -\frac{M t^2}{2} \right).
\]
\end{lemma}

\begin{proof}
\label{app:bounded-difference-proof}    
We prove the bounded difference property with respect to the first argument; the proofs for the remaining coordinates follow analogously. 

Consider two datasets that differ only in their first element:
\[
\mathbf{X} = (X_1 = a, X_2, X_3, \dots, X_N) \quad \text{and} \quad \mathbf{X}' = (X_1' = b, X_2, X_3, \dots, X_N),
\]
where \( a \neq b \). Let \( \hat{\bp} \) denote the empirical distribution based on \( \mathbf{X} \), and \( \hat{\bq} \) the empirical distribution based on \( \mathbf{X}' \). Since the two datasets differ only at position \(1\), the empirical distributions \( \hat{\bp} \) and \( \hat{\bq} \) will differ only in the entries corresponding to \(a\) and \(b\). We can therefore compute the difference in the function values based on the differences in these two entries as follows:
\begin{align*}
    |L(\mathbf{X})-L(\mathbf{X}')| &= |L(\hat{\bp})-L(\hat{\bq})| \\
    &= |\hat{p}_a(1-\hat{p}_a) + \hat{p}_b(1-\hat{p}_b) - \left(\hat{q}_a(1-\hat{q}_a) + \hat{q}_b(1-\hat{q}_b)\right)| \\
    &\leq |\hat{p}_a(1-\hat{p}_a) - \hat{q}_a(1-\hat{q}_a)| + |\hat{p}_b(1-\hat{p}_b) - \hat{q}_b(1-\hat{q}_b)| \\
    &= |\hat{p}_a - \hat{q}_a - (\hat{p}_a^2 - \hat{q}_a^2)| + |\hat{p}_b - \hat{q}_b - (\hat{p}_b^2 - \hat{q}_b^2)|\\
    &\leq |\hat{p}_a - \hat{q}_a|.\underbrace{|1-(\hat{p}_a+\hat{q}_a)|}_{\leq 1} + |\hat{p}_b - \hat{q}_b|.\underbrace{|1-(\hat{p}_b+\hat{q}_b)|}_{\leq 1} \\
    &\overset{(1)}{\leq} \underbrace{|\hat{p}_a - \hat{q}_a|}_{\leq \frac{1}{M}} + \underbrace{|\hat{p}_b - \hat{q}_b|}_{\leq \frac{1}{M}}\\
    &\overset{(2)}{\leq} \frac{2}{M}
\end{align*}

Since \(\hat{p}_a, \hat{q}_a \in [0,1]\), it follows that \(|1 - (\hat{p}_a + \hat{q}_a)| \leq 1\). The same argument applies to the term involving \(b\). This yields inequality \((1)\). 

Inequality \((2)\) uses the fact that changing one data point in the dataset affects at most two entries in the empirical distribution: one entry increases by \(1/M\), and another decreases by \(1/M\). Therefore,
\[
|\hat{p}_a - \hat{q}_a| \leq \frac{1}{M}, \qquad |\hat{p}_b - \hat{q}_b| \leq \frac{1}{M}.
\]
\end{proof}

\subsection{Proof of Theorem \ref{thm: bounding the estimation error}}
\label{app:bounding-the-estimation-error}    
To bound the estimation, we can combine the previous results as follows

\begin{align*}
    \mathbb{P}[|L(\bp)-L(\hat{\bp})| \geq t] &\leq \mathbb{P}[|L(\hat{\bp})-\mathbb{E}[L(\hat{\bp})]| + \underbrace{|L(\bp) - \mathbb{E}[L(\hat{\bp})]|}_{\overset{\mathrm{lemma}~\ref{lemma: bias term}}{\leq} \frac{1}{M}} \geq t] \\
    & \leq \mathbb{P}[|L(\hat{\bp})-\mathbb{E}[L(\hat{\bp})]| \geq t-\frac{1}{M}] \\
    & \overset{\mathrm{lemma}~\ref{lemma: bounded difference property}}{\leq}  2\exp\left( -\frac{M}{2} \left(t-\frac{1}{M}\right)^2\right)
\end{align*}

Setting \(t = \frac{1}{M} + \sqrt{\frac{2}{M}\log(\frac{2}{\epsilon})}\) implies

\begin{align}
    \mathbb{P}\left[|L(\bp)-L(\hat{\bp})| \geq \frac{1}{M} + \sqrt{\frac{2}{M}\log(\frac{2}{\epsilon})}\right] \leq \epsilon.
\end{align}
\qed

\section{Implementation and Experimental Details}
\label{app:implementation-details}

\subsection{Pseudocode for the Idealized Setting}
The pseudocode for our adaptive intervention selection strategy, assuming an ideal ICA oracle, is shown in Algorithm~\ref{alg:greedy-intervention}. The method begins by extracting an observational estimate of the causal matrix, which is assumed to be a perfectly permuted and scaled version of the true matrix. It then builds a bipartite graph representing the equivalence class of candidate graphs. At each iteration, the algorithm greedily selects the intervention expected to eliminate the largest number of inconsistent graphs by sampling perfect matchings. Once an intervention is performed, the corresponding matching edge is identified, and the process continues until the causal graph is uniquely identified or the budget is exhausted.

\begin{algorithm}[h]
\caption{Adaptive Experiment Design for Cyclic LSCMs (Idealized Setting)}
\label{alg:greedy-intervention}
\begin{algorithmic}[1]
\State \textbf{Input:} Observational distribution, budget \( K \)
\State \textbf{Output:} Reduced equivalence class after \( K \) adaptive interventions

\State Run \textbf{ideal ICA} on observational data to estimate \( I - W_{\mathrm{ICA}} \)
\State Construct the bipartite graph from nonzero entries of \( I - W_{\mathrm{ICA}} \)

\For{$t = 1$ to $K$}
    \State Check if the bipartite graph admits a \textbf{unique perfect matching}; if so, \textbf{terminate}
    \State Sample \( M_t \) perfect matchings from the bipartite graph using a sampler oracle
    \State Estimate the normalized marginal benefit \( \hat{B}(j) = \sum_{i} \hat{p}_{ij}(1 - \hat{p}_{ij}) \) for each variable \( j \)
    \State Select variable \( j^\star \) with maximum estimated benefit and intervene on it
    \State Run \textbf{ideal ICA} on the interventional distribution to estimate \( I - W_{\mathrm{ICA}}^{(j^\star)} \)
    \State Recover the true matching edge \( (r_{\pi^{-1}(j^\star)}, j^\star) \) by comparing pre- and post-intervention estimates.
    \State Remove \( r_{\pi^{-1}(j^\star)} \), \( j^\star \), and their connecting edge from the bipartite graph
\EndFor
\end{algorithmic}
\end{algorithm}

\subsection{Experimental Details for the Idealized Setting}
\label{app:implementation-details-idealized setting}
This section provides additional implementation details for the simulations presented under the ideal ICA assumption in the main text. The intervention sampler, averaging methodology, and results on dense graphs are described below.

\paragraph{Intervention Sampler.}
To estimate the marginal benefit of each candidate intervention, we rely on random samples of perfect matchings in the bipartite graph. When exact enumeration is computationally infeasible, we use a greedy heuristic to generate random (but not uniform) perfect matchings. Starting from an empty matching, the sampler proceeds by iteratively selecting an unmatched right-node \( v \) that has the minimum degree, choosing an adjacent unmatched left-node \( u \) uniformly at random, and adding the pair \( (u, v) \) to the matching. Then we remove \( v \), \( u \), and all their incident edges from the bipartite graph. This process continues until a perfect matching is formed. One can show that this method always finds a random perfect matching if one exists. While this sampling scheme is non-uniform, it is scalable and yields accurate marginal benefit estimates in practice.

\paragraph{Averaging Across Trials.}
For each fixed number of nodes \( n \), we repeat the simulation over 60 random graph instances to mitigate randomness in graph generation and intervention ordering. We report the \emph{average number of interventions} required for full identification and visualize variability using \emph{standard deviation bars} around the mean in our plots in Figure~\ref{fig:strategy-comparison-appendix}. These error bars demonstrate the stability of each strategy across different random seeds.

\paragraph{Experiments on Dense Graphs.}
In the main experiments, we used Erdős–Rényi random graphs with edge probability \( p = c/n \), ensuring that the expected degree remains constant as \( n \) increases. In this regime, the feedback vertex set (FVS) can be computed exactly, and serves as a theoretical lower bound. To further stress-test the methods, we also evaluate performance on \emph{dense} graphs where each edge is included independently with a fixed probability (e.g., \( p = 0.2 \)). These dense graphs often contain an exponential number of cycles, making exact FVS computation infeasible.

Nonetheless, we compare the performance of our adaptive method against Random and Max-Degree baselines in these settings. The results, shown in Figure~\ref{fig:comparison-sample-dense-graph}, reveal that our algorithm continues to outperform these heuristics even in dense regimes where the graph structure is significantly more complex. For full reproducibility, including code and data used in all experiments, we refer to the repository available at \url{https://github.com/EhsanSharifian/exp-design-cyclic-models}.

\begin{figure}[t!]
    \centering
    \begin{subfigure}[t]{\linewidth}
        \centering
        \includegraphics[width=0.85\linewidth]{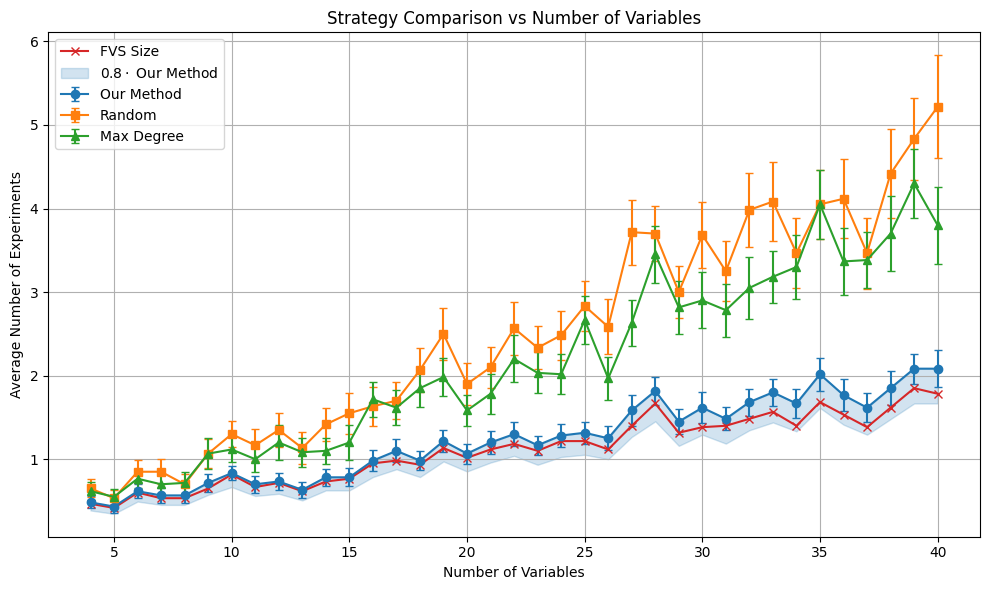}
        \caption{Sample mode on sparse graphs with a maximum of 40 nodes.}
        \label{fig:comparison-sample-error-bars}
    \end{subfigure}
    \vspace{1em} 
    \begin{subfigure}[t]{\linewidth}
        \centering
        \includegraphics[width=0.85\linewidth]{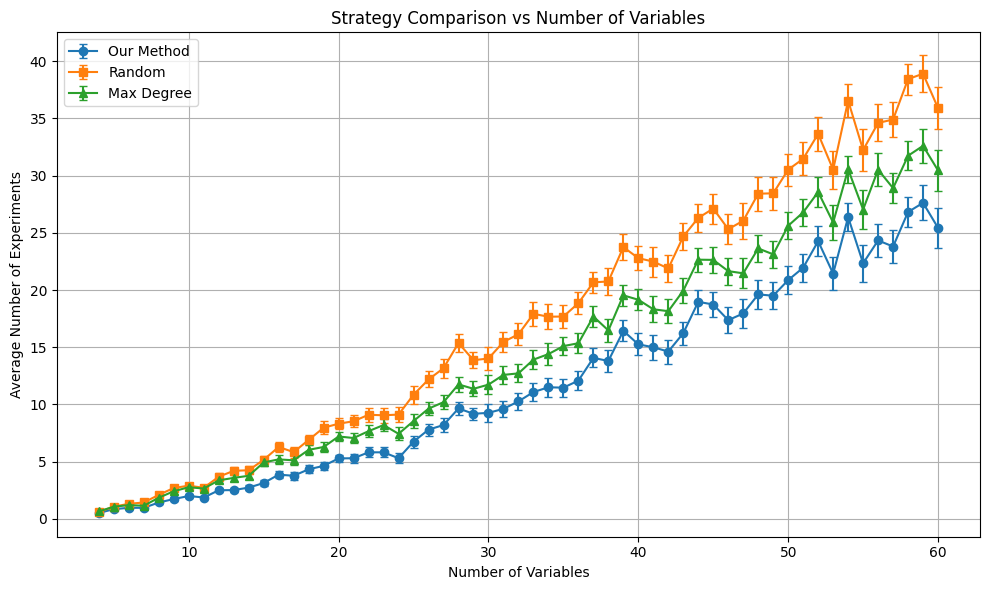}
        \caption{Sample mode on dense graphs with a maximum of 60 nodes.}
        \label{fig:comparison-sample-dense-graph}
    \end{subfigure}
    \caption{Comparison of intervention strategies under the ideal ICA assumption with a sample-based matching sampler. Error bars indicate standard deviation over trials.}
    \label{fig:strategy-comparison-appendix}
\end{figure}

\subsection{Implementation in the Finite-Sample Setting}
\label{app:Implementation-with-Finite-Sample}
This section provides the implementation details for the finite-sample setting, where we use the \texttt{FastICA} algorithm on finite samples instead of having access to an ideal oracle. This introduces estimation noise, necessitating a more robust algorithm, presented in Algorithm~\ref{alg:greedy-ica}.

\begin{algorithm}[h!]
\caption{ICA-based Adaptive Experiment Design for Cyclic LSCMs}
\label{alg:greedy-ica}
\begin{algorithmic}[1]
\State \textbf{Input:} Observational data, intervention budget \( K \)
\State \textbf{Output:} Recovered causal matrix $\widehat{W}$

\State Run \texttt{FastICA} on the observational data to estimate \( I - W_{\mathrm{ICA}} \)
\State Normalize and threshold \( I - W_{\mathrm{ICA}} \) adaptively to ensure matchability
\State Construct a bipartite graph from the support of the thresholded matrix

\For{$t = 1$ to $K$}
    \State Check for a \textbf{unique perfect matching}; if found, \textbf{terminate}
    \State Sample \( M_t \) perfect matchings from the bipartite graph
    \State Compute normalized marginal benefit \( \hat{B}(j) = \sum_{i} \hat{p}_{ij}(1 - \hat{p}_{ij}) \) for each variable \( j \)
    \State Select the variable \( j^\star \) with maximum estimated benefit and intervene on it
    \State Run \texttt{FastICA} on the interventional distribution to estimate \( I - W_{\mathrm{ICA}}^{(j^\star)} \)
    \State Rank candidate matching edges \( (r_i, j^\star) \) by comparing rows of pre- and post-intervention matrices
    \For{each candidate edge \( (r_i, j^\star) \) in ranked order}
        \State Temporarily remove \( r_i \) and \( j^\star \) from the bipartite graph
        \If{the remaining graph admits a perfect matching}
            \State Accept \( (r_i, j^\star) \) as the matching edge
            \State Remove \( r_i \), \( j^\star \), and their connecting edge from the actual graph
            \State \textbf{Break}
        \EndIf
    \EndFor
    \If{no valid edge found}
        \State \textbf{Terminate} with failure
    \EndIf
\EndFor
\end{algorithmic}
\end{algorithm}

The algorithm incorporates two key modifications to handle noise. First, after running ICA, it performs \textit{adaptive thresholding} (Line~4): it begins with strict thresholds to remove noise and progressively relaxes them only until the resulting bipartite graph admits at least one perfect matching. Second, after an intervention, it uses a \textit{safe matching} procedure to identify the correct row (Lines~13-20). Instead of picking the single best match directly, it ranks candidate rows by similarity and iterates through them. It accepts the first candidate that, when matched, leaves the remaining graph in a state where a perfect matching is still possible. This conservative strategy ensures robustness against estimation errors.

\begin{figure}[t!]
    \centering
    \includegraphics[width=0.85\linewidth]{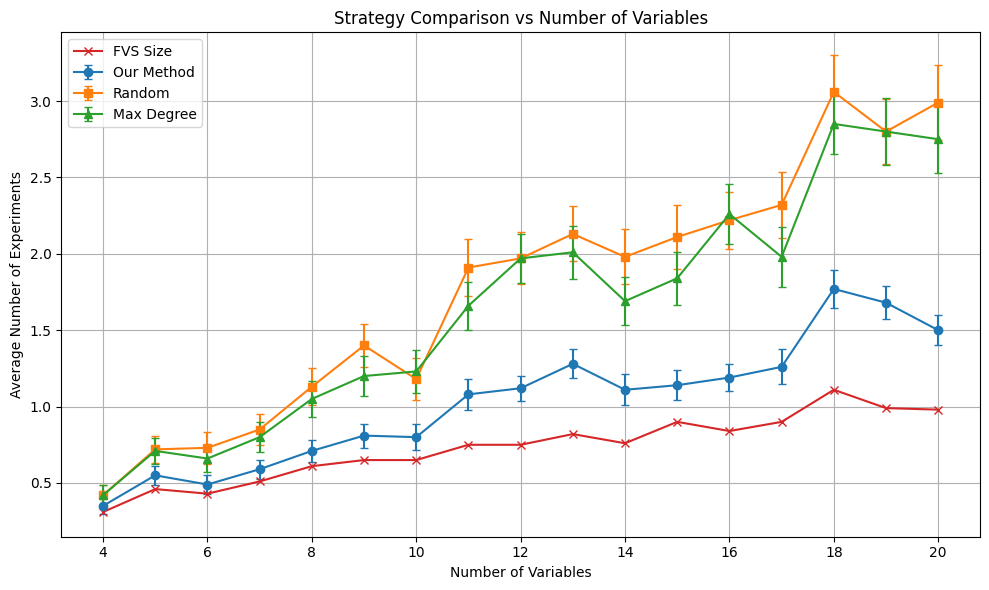}
    \caption{Comparison of intervention strategies under finite-sample ICA with a sample-based matching sampler. Error bars represent standard deviation over trials.}
    \label{fig:strategy-comparison-FastICA}
\end{figure}

The performance of this robust algorithm is detailed in Figure~\ref{fig:strategy-comparison-FastICA}. The results show that our adaptive strategy continues to outperform the Random and Max Degree baselines, demonstrating its practical effectiveness. Furthermore, Table~\ref{tab:relative-errors-appendix} quantifies the accuracy of the final recovered matrix $\widehat{W}$ against the ground truth $W^*$, demonstrating high-fidelity recovery.

\begin{table}[h!]
\centering
\caption{Distribution of relative errors ($\varepsilon_{\mathrm{rel}}=\frac{\left\|\widehat{W}-W^*\right\|_F}{\left\|W^*\right\|_F}$) for the recovered matrix $\widehat{W}$ across all experimental runs in the finite-sample setting.}
\label{tab:relative-errors-appendix}
\begin{tabular}{@{}lc@{}}
\toprule
\textbf{Relative Error Range} & \textbf{Number of Instances} \\ \midrule
$< 0.05$            & 860 \\
$0.05$ -- $0.15$    & 275  \\
$0.15$ -- $0.25$    & 95  \\
$> 0.25$            & 470 \\ \bottomrule
\end{tabular}
\end{table}

\section{Details on Extensions}
\label{app:extensions}

\subsection{Multi-Node Interventions}
\label{app:extensions-multi-node}
The adaptive framework presented in the main text focuses on single-node interventions but can be naturally extended to handle simultaneous interventions on multiple nodes. This extension allows for the parallelization of experiments, which can significantly improve the efficiency of the causal discovery process. Here, we provide a detailed analysis of this generalization.

\paragraph{Formalizing the Multi-Node Intervention.}
Suppose an experiment is conducted where hard interventions are simultaneously applied to a set of $t$ variables, denoted by $E = \{i_1, i_2, \dots, i_t\}$. A hard intervention on a variable effectively removes all its incoming causal links. The post-intervention weight matrix, $W^{(E)}$, is therefore obtained by setting the rows corresponding to the indices in $E$ to zero. Formally, for a row index $k$:
\[
(W^{(E)})_{k, \cdot} = 
\begin{cases}
    \mathbf{0} & \text{if } k \in E \\
    W_{k, \cdot} & \text{if } k \notin E
\end{cases}
\]
Applying ICA to the observational and interventional data yields the permuted and scaled matrices $P D (I - W)$ and $P' D' (I - W^{(E)})$, respectively. The core idea is to identify structural information by comparing the set of rows from these two matrices.

\paragraph{Localizing the Permutation Ambiguity.}
The key difference between $(I - W)$ and $(I - W^{(E)})$ lies precisely in the $t$ rows corresponding to the interventions. When we compare the set of $n$ rows recovered from the observational data with the $n$ rows from the interventional data, we will find that:
\begin{itemize}
    \item $n-t$ rows from the observational set have a corresponding match in the interventional set (up to scaling). These are the rows associated with the non-intervened variables $V \setminus E$.
    \item $t$ rows from the observational set have no match. These "unmatched" rows must correspond to the variables in the intervened set $E$.
\end{itemize}
This single multi-node intervention effectively partitions the global row-permutation problem. It breaks the cycle reversion ambiguity that might exist between the nodes in $E$ and the nodes in $V \setminus E$. We now know with certainty that the $t$ unmatched observational rows must be assigned to the $t$ true row indices corresponding to the variables in $E$. 

The global challenge of finding one permutation of size $n$ is thus reduced to two smaller, independent subproblems: one for the set $E$ and one for the set $V \setminus E$. The specific assignment of the $t$ identified rows to the $t$ variables within $E$ remains to be solved, a subproblem that is addressed by the acyclicity condition discussed next.

\paragraph{Condition for Full Identification within the Intervened Set.}
While a multi-node intervention is always informative for localizing ambiguity, we can achieve an even stronger result (full identification of all intervened nodes) if the intervention set $E$ is chosen carefully. Specifically, if the subgraph induced by the variables in $E$ is \textit{acyclic}, the internal permutation ambiguity within $E$ can be fully resolved.

A crucial subtlety is that we cannot verify the acyclicity of an \textit{arbitrary} set of variables directly from the observational matrix $I - W_{\mathrm{ICA}}$, due to the row-permutation ambiguity we aim to solve. However, a \textbf{verifiable sufficient condition} can be considered herein. The \textbf{Strongly Connected Components (SCCs)} of the causal graph are identifiable from the observational data, as the component structure is invariant across all graphs in the permutation-equivalence class.

Therefore, a verifiable strategy is to select an intervention set $E$ such that each variable $i_k \in E$ belongs to a different SCC. By definition, nodes residing in distinct SCCs cannot form a cycle among themselves. This choice guarantees that the subgraph induced by $E$ is acyclic.

Consequently, when we intervene on such a set and identify the $t$ unmatched observational rows, there will be only one way to assign these rows to the variables in $E$ that is consistent with an acyclic structure. This uniquely reveals the true row for each intervened variable and fully resolves the local ambiguity.

\paragraph{Implications for Experimental Design.}
This extension enables a more efficient, parallelized experimental design strategy. Instead of selecting one node at a time, the algorithm could identify maximal subsets of variables that belong to different SCCs, and intervene on them simultaneously. This would remarkably reduce the number of experiments required for full graph identification.

\subsection{Imperfect and Noisy Interventions}
\label{app:extensions-imperfect}
The core framework described in this paper assumes ``perfect'' or ``hard'' interventions, where all incoming causal mechanisms for a target variable are removed (i.e., the corresponding row in the weight matrix $W$ is set to zero). However, the method's validity is not restricted to this idealized case. It can be extended to a broad class of imperfect or ``soft'' interventions, which are often more representative of real-world experiments.

\paragraph{The General Identifiability Condition.}
The fundamental mechanism of our approach is the identification of the intervened variable by detecting a change in its corresponding structural equation. This is achieved by comparing the set of rows recovered from observational data, $R_{obs} = \{\text{rows of } (I-W)\}$, with the set of rows recovered from interventional data, $R_{int}$. An intervention on a variable $X_k$ is considered informative as long as it sufficiently perturbs the $k$-th row of the causal matrix.

Let the imperfect intervention on $X_k$ change its corresponding row in the weight matrix from $W_{k, \cdot}$ to a new row $W'_{k, \cdot}$. The new post-intervention matrix is $(I - W')$. The intervention is guaranteed to be identifiable if the new, perturbed row $(I - W')_{k, \cdot}$ is not a scaled version of any of the \textit{other} original rows from the observational matrix. That is, for all $j \neq k$:
$$
(I - W')_{k, \cdot} \neq c \cdot (I - W)_{j, \cdot} \quad \text{for any scaling factor } c.
$$
If this condition holds, the original row $(I - W)_{k, \cdot}$ will be the only row in $R_{obs}$ that has no corresponding match in $R_{int}$, uniquely revealing $k$ as the target of the intervention.

This identifiability condition holds with high probability for interventions that alter the causal mechanism of a variable. The chances of a random modification creating a row that perfectly aligns with an existing one are negligible. 

\subsection{Generalization to Non-linear Models}
\label{app:extensions-non-linear}
While this work focuses on linear models for their tractability and widespread use, the core principles of our adaptive design framework (representing ambiguity via a bipartite graph and greedily selecting interventions to reduce it) can be extended to non-linear Structural Causal Models (SCMs). This generalization relies on recent advances in non-linear Independent Component Analysis (ICA) and the structural information revealed by the Jacobian of the causal mechanism.

\paragraph{Causal Discovery via the Jacobian.}
Consider a general non-linear SCM defined by the structural equations $X_i = f_i(\text{Pa}_i, N_i)$, where $\text{Pa}_i$ are the parents of node $i$ and $N_i$ are independent noise terms. Let $\mathbf{f}$ be the vector-valued function that maps the exogenous noises $\mathbf{N}$ to the observed variables $\mathbf{X}$. The inverse of this map, $\mathbf{f}^{-1}$, recovers the noises from the variables, such that $\mathbf{N} = \mathbf{f}^{-1}(\mathbf{X})$.

This formulation is a direct generalization of the linear model. In our linear case, the relationship is $\mathbf{N} = (I-W)\mathbf{X}$, meaning the matrix $(I-W)$ itself is the (linear) inverse map. For the non-linear case, a key result from \cite{reizinger2023jacobian} for acyclic graphs shows that the \textit{Jacobian of the inverse map}, $\nabla \mathbf{f}^{-1}(\mathbf{X})$, plays the analogous role. Specifically, its support (the pattern of non-zero entries) is equivalent to the support of $I-A$, where $A$ is the graph's adjacency matrix. We conjecture that this result can be extended to cyclic models under similar identifiability conditions.

\paragraph{Role of Non-linear ICA.}
The theoretical connection between the Jacobian and the graph structure can be operationalized using modern non-linear ICA algorithms \cite{khemakhem2020variational, kivva2022identifiability}. These methods can learn the inverse map $\mathbf{f}^{-1}$ from observational data alone, up to known indeterminacies (such as permutation, scaling, and element-wise monotonic transformations). By computing the Jacobian of the learned function $\hat{\mathbf{f}}^{-1}$, we can obtain an estimate of the matrix $I-A$ up to a row permutation and scaling. Let us denote this recovered matrix by $(I - A)_{\text{ICA}}$.

\paragraph{Adapting the Experimental Design Framework.}
The recovered matrix $(I-A)_{\text{ICA}}$ serves the exact same purpose in the non-linear setting as $(I-W_{\text{ICA}})$ does in our linear framework. Its support can be used to construct a bipartite graph that perfectly characterizes the permutation ambiguity of the observational equivalence class. From this representation, we can identify the SCCs and apply the same adaptive intervention strategy to resolve the ambiguity. An intervention on a node $X_k$ would modify its function $f_k$, which in turn would alter the $k$-th row of the Jacobian, allowing for its identification.

\paragraph{Current Limitation.}
The primary challenge for this non-linear extension lies in the identifiability constraints. The main structural requirement is that \textit{each node in the causal graph must have a unique parent set} to ensure the rows of the Jacobian have unique sparsity patterns to enable
correct matching in the bipartite graph. Consequently, based on current results in nonlinear ICA, it is
needed that each node has a unique parent set in order to carry out the experiment design step. We hope that this assumption can be relaxed with future advances in nonlinear ICA.

\begin{comment}
A second, more subtle challenge arises from the estimation process itself. As noted, current non-linear ICA methods recover the inverse map $\mathbf{f}^{-1}$ up to several indeterminacies, crucially including \textbf{element-wise monotonic transformations}. This means the learned function $\hat{\mathbf{f}}^{-1}$ may differ from the true one by such a transformation. Consequently, the support of the Jacobian of the learned function, $\nabla \hat{\mathbf{f}}^{-1}$, might differ from that of the true Jacobian, $\nabla \mathbf{f}^{-1}$. Further investigation is needed to characterize the precise conditions under which they are equivalent.
\end{comment}

\end{document}